\documentclass[conference]{IEEEtran}
\usepackage[letterpaper,top=0.75in,bottom=1in,left=0.625in,right=0.625in]{geometry}

\newif\ifdraft
\draftfalse

\newif\ifoc 
\ocfalse

\newif\iffull 
\fulltrue

\ifdraft
\usepackage{xcolor}
\usepackage{xargs} 
\usepackage{tikz} 
\usepackage[textsize=footnotesize]{todonotes}
\newcommandx{\sw}[2][1=]{\todo[linecolor=blue,
			backgroundcolor=blue!10,bordercolor=blue,#1]{{\bf SF}: #2}}
\newcommandx{\jy}[2][1=]{\todo[linecolor=red,
			backgroundcolor=red!10,bordercolor=red,#1]{{\bf JY}: #2}}
\newcommandx{\sh}[2][1=]{\todo[linecolor=orange,
			backgroundcolor=orange!10,bordercolor=orange,#1]{{\bf SH}: #2}}
\else
\def\sw#1{}
\def\jy#1{}
\def\sh#1{}
\fi

\usepackage{tabularx}
\usepackage{amssymb,amsmath,amsthm}
\usepackage[mathcal]{euscript}
\usepackage{textcomp}
\usepackage{graphicx}
\usepackage{url}
\usepackage{wrapfig}
\usepackage{enumerate}
\usepackage{soul}
\usepackage{tipa}
\usepackage{cite}
\usepackage[percent]{overpic}
\usepackage{tikz}
\usepackage{pgfgantt}
\usepackage{xspace}
\usepackage[linesnumbered,ruled]{algorithm2e}
\usepackage{diagbox}
\usepackage{multirow}

\usepackage[skip=12pt,font=scriptsize]{caption}
\newtheorem{theorem}{Theorem}

\newtheorem{proposition}[theorem]{Proposition}

\newtheorem{lemma}[theorem]{Lemma}
\newtheorem{problem}{Problem}

\usepackage{mathtools,xparse}

\makeatletter
\newcommand{\customlabel}[2]{%
\protected@write \@auxout {}{\string \newlabel {#1}{{#2}{}}}}
\makeatother

\usepackage{xspace}
\def\R{\mathcal R}
\def\C{\mathcal C}

\def\P{\mathcal P}

\def\W{\mathcal W}
\def\opg{{\sc {OPG}}\xspace}

\newenvironment{remark}{\noindent\textbf{Remark.} }{\hfill$\triangle$}

\iffull
{}
\else
\fi

\title{Efficient Algorithms for Optimal Perimeter Guarding}

\author{\authorblockN{Si Wei Feng\authorrefmark{1}, 
Shuai D. Han\authorrefmark{1}, 
Kai Gao\authorrefmark{2} and Jingjin Yu\authorrefmark{1}}
\authorblockA{\authorrefmark{1}Department of Computer Science, 
Rutgers University at New Brunswick, Piscataway, New Jersey, U.S.A.  
}
\authorblockA{\authorrefmark{2}
School of Mathematics, University of Science and Technology of China,
Hefei, Anhui, China
}
}

\begin{document}
\maketitle
\begin{abstract}We investigate the problem of optimally assigning 
a large number of robots (or other types of autonomous agents) to 
guard the perimeters of closed 2D regions, where the perimeter 
of each region to be guarded may contain multiple disjoint polygonal
chains. Each robot is responsible for guarding a subset of a 
perimeter and any point on a perimeter must be guarded by some robot. 
In allocating the robots, the main objective is to minimize the 
maximum 1D distance to be covered by any robot along the boundary 
of the regions. For this optimization problem which we call optimal 
perimeter guarding (\opg), thorough structural analysis is performed, 
which is then exploited to develop fast exact algorithms that run in 
guaranteed low polynomial time. In addition to formal analysis and proofs, 
experimental evaluations and simulations are performed that further 
validate the correctness and effectiveness of our algorithmic results. 
\end{abstract}

\section{Introduction}\label{section:introduction}
Consider the scenario from Fig.~\ref{fig:example}, which contains a 
closed region with its boundary or border demarcated by the red 
and dotted blue polygonal chains (p-chains for short). To 
secure the region, either from intrusions from the outside or unwanted 
escapes from within, it is desirable to deploy a number of autonomous 
robots to monitor or guard either the entire boundary or selected 
portions of it, e.g., the three red p-chains), with each robot 
responsible for a continuous section. Naturally, one might also want 
to have an even coverage by the robots, e.g., minimizing the maximum 
effort from any robot. In practice, such effort may correspond to 
sensing ranges or motion capabilities of robots, which are always 
limited. As an intuitive example, the figure may represent the top 
view of a castle with its entire boundary being a high wall on which 
robots may travel. The portion of the wall marked with the three 
red p-chains must be protected whereas the part marked by the dotted 
blue p-chains may not need active monitoring (e.g., the outside of which 
may be a cliff or a body of deep water). The green and orange p-chains 
show an optimal distribution of the workload by $8$ robots that covers 
all red p-chains but skips two of the three blue dotted p-chains.

More formally we study the problem of deploying a large 
number of robots to guard a set of 1D {\em perimeters}. 
Each perimeter is comprised of one or more 1D (p-chain) {\em segments} 
that are part of a circular boundary (e.g., the red p-chains in 
Fig.~\ref{fig:example}). Each robot is tasked to guard a continuous 1D 
p-chain that {\em covers} a portion of a perimeter. As the main 
objective, we seek an allocation of robots such that {\em (i)} the 
union of the robots' coverage encloses all perimeters and {\em (ii)} 
the maximum coverage of any robot is minimized. We call this 1D 
deployment problem the Optimal Perimeter Guarding (\opg) problem. 
\begin{figure}[ht]
\iffull \vspace*{2mm} \fi
\begin{center}
\begin{overpic}[width={\iffull 3.5in \fi},tics=5]{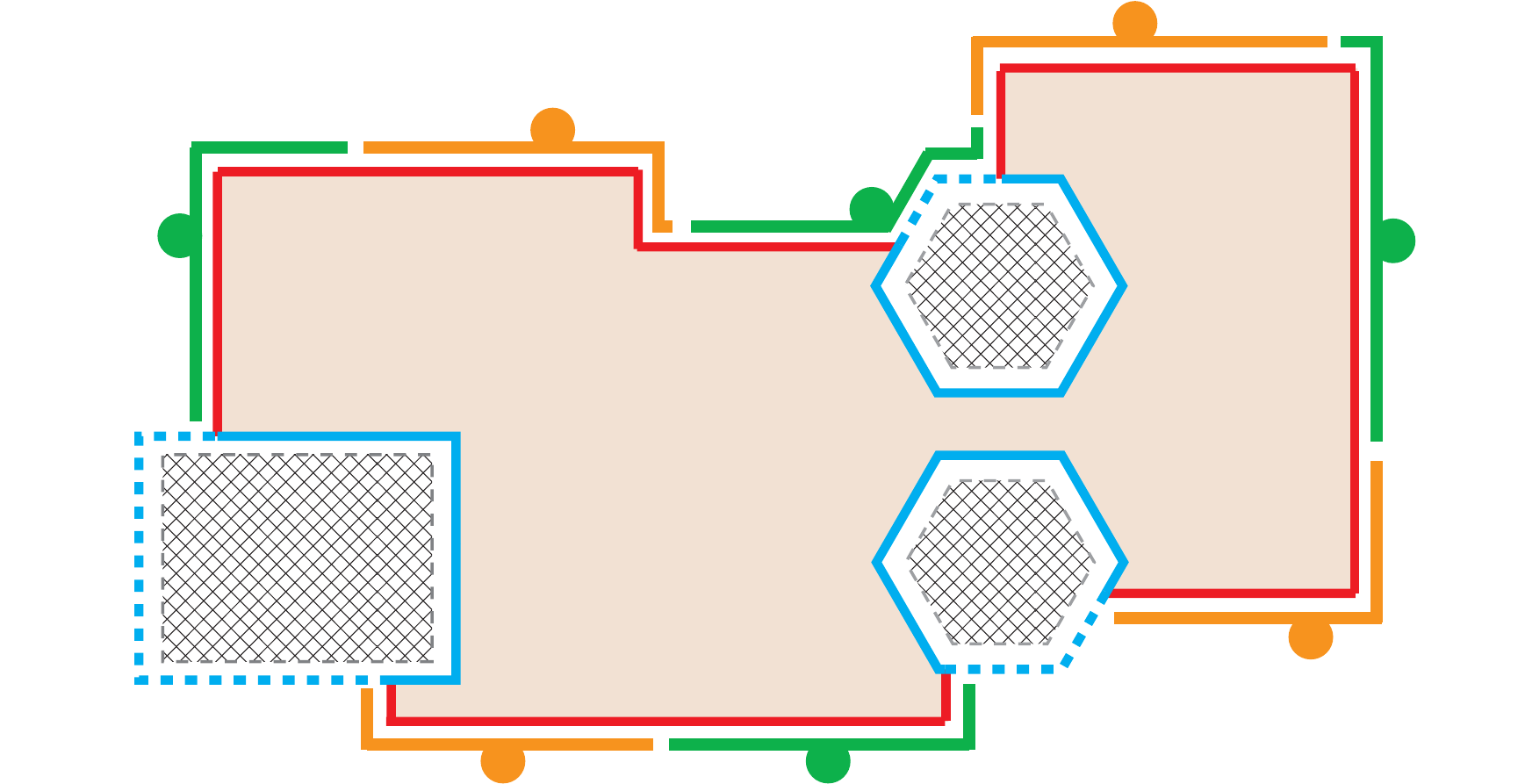}
\end{overpic}
\end{center}
\iffull {} \fi
\caption{\label{fig:example} An illustrative scenario where a perimeter, 
in this case represented as the red polygonal chains (p-chains), must be 
guarded by $n = 8$ robots, which are constrained to only travel along 
the perimeter boundary (the red p-chains plus the dotted blue ones, which 
are gaps that do not need to be guarded). An optimal set of locations for 
the $8$ robots and the coverage region for each robot are marked on the 
perimeter boundary in green and orange, which minimizes the maximum coverage 
required for any robot.}
\iffull {} \fi
\end{figure}

In this work, three main \opg variants are examined. The settings 
regarding the perimeter in these three variants are: {\em (i)}
multiple perimeters with each having a single connected component; 
{\em (ii)} a single perimeter containing multiple connected components; 
and {\em (iii)} multiple perimeters with each containing multiple 
connected components (the most general case). For all three variants, 
we have developed exact algorithms for solving \opg that runs in low
polynomial time. More specifically, let there be $n$ robots, $m$ 
perimeters, with perimeter $i$ ($1 \le i \le m$) containing $q_i$ 
connected components. If $m = 1$, then let the only perimeter contains
$q$ connected components. For the three variants, our algorithm 
computes an optimal solution in time $O(m(\log n + \log m) + n)$, 
$O(q^2\log(n+q) + n)$, and $O((\sum_{1\le i \le m} q_i^2) \log(n + 
\sum_{1\le i \le m} q_i) + n)$, respectively, which
are roughly quadratic in the worst case. 
The modeling of the \opg problem and the development of the efficient 
algorithms for \opg constitute the main contribution of this paper. 

With an emphasis on the deployment of a large number of robots, 
within multi-robot systems research 
\cite{arai2002advances,gerkey2004formal,ren2008distributed,bullo2009distributed}, 
our study is closely related to {\em formation control}, e.g., 
\cite{ando1999distributed,jadbabaie2003coordination,olfati2004consensus,ren2005consensus,cheng2008almost,mesbahi2010graph,yu2012rendezvous},
where the goal is to achieve certain {\em distributions} through 
continuous (often, local sensing based) interactions among the 
agents or robots. Depending on the particular setting, the 
distribution in question may be spatial, e.g., rendezvous
\cite{ando1999distributed,yu2012rendezvous}, or maybe an agreement 
in agent velocity is sought \cite{jadbabaie2003coordination,ren2005consensus}. 
In these studies, the resulting formation often has some 
degree-of-freedoms left unspecified. For example, rendezvous 
results \cite{ando1999distributed,yu2012rendezvous} often come 
with exponential convergence guarantee, but the location of
rendezvous is generally unknown {\em a priori}. 

On the other hand, in multi-robot task and motion planning problems (e.g.,
\cite{smith2009monotonic,ayanian2010decentralized,liu2011multi,liu2013optimal,turpin2014goal,turpin2014capt,alonso2015multi,SolYu15}), 
especially ones with a {\em task allocation} element 
\cite{smith2009monotonic,liu2011multi,liu2013optimal,turpin2014goal,turpin2014capt,SolYu15},
the (permutation-invariant) target configuration is often mostly 
known. The goal here is finding a one-to-one mapping between individual 
robots and the target locations (e.g., deciding a {\em matching}) and 
then plan (possibly collision-free) trajectories for the robots to reach 
their respective assigned targets \cite{turpin2014goal,turpin2014capt,SolYu15}.  
In contrast to formation control and multi-robot motion planning research, 
our study of \opg seeks to determine an exact, optimal distribution 
pattern of robots (in this case, over a fairly arbitrary, bounded 1D 
topological domain). Thus, solutions to \opg may serve as the target 
distributions for multi-robot task and motion planning, which is the 
main motivation behind our work. The generated distribution pattern is 
also potentially useful in multi-robot persistent monitoring 
\cite{soltero2014decentralized} and coverage \cite{howard2002mobile,schwager2009optimal} 
applications, where robots are asked to carry out sensing tasks in some 
optimal manner. 
	
As a multi-robot coverage problem, \opg is intimately connected to Art 
Gallery problems \cite{o1987art,shermer1992recent}, with origins traceable 
to half a century ago \cite{klee1969every}. Art Gallery problems assume 
a visibility-based \cite{lozano1979algorithm} sensing model; in a typical 
setup \cite{o1987art}, the {\em interior} of a polygon must be visible to at 
least one of the guards, which may be placed on the boundaries, corners, 
or the interior of the polygon. Finding the optimal number of guards are 
often NP-hard \cite{lee1986computational}. Alternatively, disc-based sensing 
model may be used, which leads to the classical {\em packing} problem 
\cite{thue1910dichteste,hales2005proof}, where no overlap is allowed between 
the sensors' coverage area, the {\em coverage} problem 
\cite{drezner1995facility,cortes2004coverage,pavone2009equitable,martinez2010distributed,pierson2017adapting}, 
where all 
workspace must be covered with overlaps allowed, or the {\em tiling} problem 
\cite{kershner1968paving}, where the goal is to have the union of sensing
ranges span the entire workspace without overlap. For a more complete 
account on Art Gallery, packing, and covering, see Chapters 2, 3, and 33 of
\cite{toth2017handbook}. Despite the existence of a large body of literature 
performing extensive studies on these intriguing computational geometry 
problems, these types of research mostly address domains that are 2D and 
higher. To our knowledge, \opg, as an optimal coverage problem over 
a non-trivial 1D topological space, represents a practical and novel formulation 
yet to be fully investigated. 

The rest of the paper is organized as follows. 
The \opg problem and some of its most basic properties are described 
in Section~\ref{section:problem}. 
In Section~\ref{section:analysis}, a thorough structural analysis of \opg 
with single and multiple perimeters is performed, paving the way for 
introducing the full algorithmic solutions in Section~\ref{section:algorithm}.
Then, in Section~\ref{section:evaluation}, comprehensive numerical 
evaluations of the multiple polynomial-time algorithms are carried out. 
In addition, two realistic application scenarios are demonstrated. 
In Section~\ref{section:conclusion}, we conclude with  
additional discussions.

\section{The Optimal Perimeter Guarding Problem}\label{section:problem}
Let $\W \subset \mathbb R^2$ be a compact (i.e., closed and bounded) 
two-dimensional workspace. There are  $m$ pairwise disjoint {\em 
regions} $\R = \{R_1, \ldots, R_m\}$ where each region $R_i \subset \W$ 
is homeomorphic to the closed unit disc, i.e., there exists a continuous 
bijection $f_i: R_i \to \{(x, y) \mid x^2 + y^2 \le 1\}$ for all $1 \le 
i \le m$. For a given region $R_i$, let $\partial R_i$ be its (closed) 
boundary (therefore, $f_i$ maps $\partial R_i$ to the unit circle  
$\mathbb S^1$). With a slight abuse of notation, define $\partial \R 
= \{\partial R_1, \ldots, \partial R_m\}$. For each $\R_i$, $P_i \subset 
\partial R_i$ is called the {\em perimeter} of $R_i$ which is either a 
single closed curve or formed by a finite number of possibly curved line 
segments. In this paper, we assume a perimeter is given as a single p-chain (possibly
a polygon) or multiple disjoint p-chains. Let $\P = \{P_1, \ldots, P_m\}$, 
which must be {\em guarded}. More formally, each $P_i$ is 
homeomorphic to a compact subset of the unit circle. For 
a given $P_i$, each of its maximal connected component (a p-chain) is 
called a {\em perimeter segment} or {\em segment}, whereas each 
maximal connected component of $\partial R_i \backslash P_i$ is called a 
{\em perimeter gap} or {\em gap}. An example is illustrated in 
Fig.~\ref{fig:example-boundaries} with two regions. 

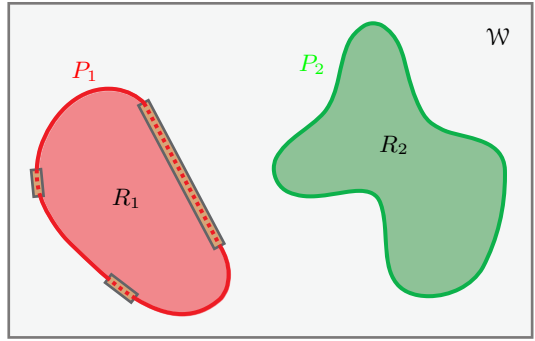
\begin{figure}[ht]
\iffull {} \fi
\begin{center}
\begin{overpic}[width={\iffull 3.5in \fi},tics=5]
{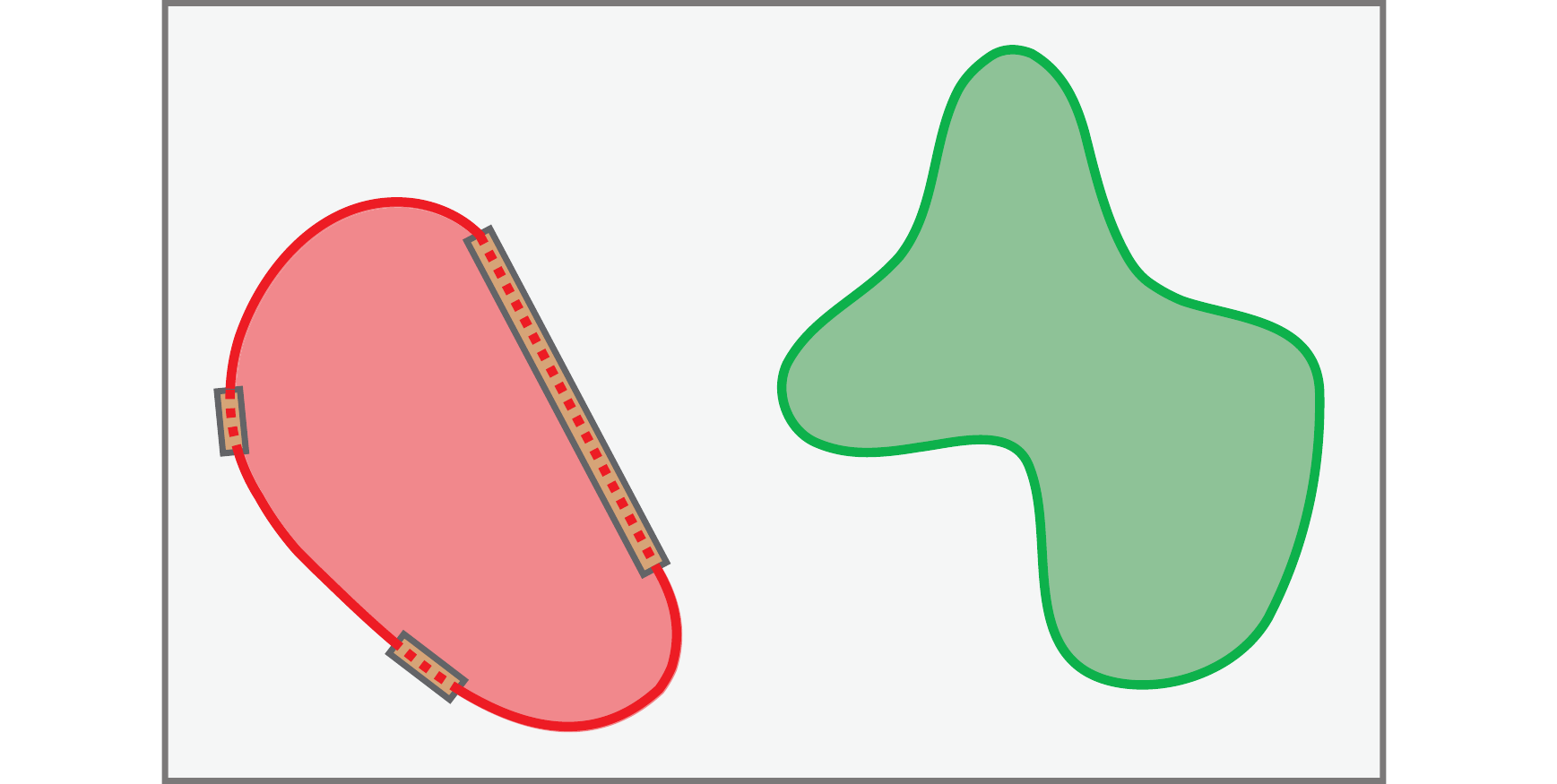}
\put(26,20){{\small $R_1$}}
\put(20,39){{\small \textcolor{red}{$P_1$}}}
\put(66,28){{\small $R_2$}}
\put(54,40){{\small \textcolor{green}{$P_2$}}}
\put(82,44){{\small $\W$}}
\end{overpic}
\end{center}
\iffull {} \fi
\caption{\label{fig:example-boundaries} An example of a workspace $\W$ 
with two regions $\{R_1, R_2\}$. Due to three {\em gaps} on $\partial R_1$, 
marked as dotted lines within long rectangles, $P_1 \subset \partial R_1$ 
has three {\em segments} (or maximal connected components); $P_2 = \partial 
R_2$ has a single segment with no gap.}
\iffull {} \fi
\end{figure}

There are $n$ indistinguishable point robots residing in $\W$.
These robots are to be deployed to {\em cover} the perimeters $\P$ such 
that each robot $1 \le j \le n$ is assigned a continuous closed subset 
$C_j$ of some $\partial R_i, 1 \le i \le m$. All of $\P$ must be 
{\em covered} by $\C = \{C_1, \ldots, C_n\}$, i.e., 
$\bigcup_{P_i \in \P} P_i  \subset \bigcup_{C_j \in \C} C_j$,
which implies that elements of $\C$ need not intersect on their
interiors. Hence, it is assumed that any two elements of $\C$ may share 
at most their endpoints. Such a $\C$ is called a {\em cover} of $\P$. 

Given a cover $\C$, for a $C_j \in \C$, $1 \le j \le n$, let $len(C_j)$ 
denote its length (more formally, measure). It is desirable to minimize 
the maximum $len(C_j)$, i.e., the goal is to find a cover $\C$ such that 
the value 
$\max_{C_j \in \C} len(C_j)$
is minimized. This corresponds to minimizing the maximum workload for 
each robot or agent. The formal definition of the Optimal Perimeter 
Guarding (\opg) problem is provided as follows. 

\begin{problem}[Optimal Perimeter Guarding (\opg)] Given the perimeter 
$\P = \{P_1, \ldots, P_m\}$ of a set of 2D regions $\R =\{R_1, \ldots, 
R_m\}$, find a set of $n$ polygonal chains $\C^* = \{C_1^*, 
\ldots, C_n^*\}$ such that $\C^*$ covers $\P$, i.e., 
\begin{align}\label{eq:coverage}
\bigcup_{P_i \in \P} P_i  \subset \bigcup_{C_j^* \in \C^*} C_j^*,
\end{align}
with the maximum of $len(C_j^*), 1 \le j \le n$ minimized, i.e., among all 
covers $\C$ satisfying~\eqref{eq:coverage}, 
\begin{align}\label{eq:objective}
\C^* = \underset{\C}{\mathrm{argmin}} \max_{C_j \in \C} len(C_j).
\end{align}
\end{problem}

Here, we introduce the technical assumption that the ratio between 
the length of $\partial \R$ and the length of $\partial \P$ is polynomial
in the input parameters. That is, the length of $\partial \R$ is not much 
larger than the length of $\partial \P$ . 
The assumption makes intuitive sense as any 
gap should not be much larger than the perimeter in practice. We note 
that the assumption is not strictly necessary but helps simplify the 
correctness proof of some algorithms. 

Henceforth, in general, $\C^*$ is used when an optimal cover is meant 
whereas $\C$ is used when a cover is meant. We further define the optimal 
single robot coverage length as 
\begin{align}\label{eq:l-star}
\iffull {} \fi
\ell^* = \min_{\C} \max_{C_j \in \C} len(C_j).
\iffull {} \fi
\end{align}

Fig.~\ref{fig:example} shows an example of an optimal cover by $8$ robots 
of a perimeter with three components. Note that one of the three gaps 
(the one on the top area as part of the hexagon) is fully covered 
by a robot, which leads to a smaller $\ell^*$ as compared to other 
feasible solutions. This interesting phenomenon, which is 
actually a main source of the difficulty in solving \opg, is explored 
more formally in Section~\ref{section:analysis} 
(Proposition~\ref{p:max-no-exclusion}).

Given the \opg formulation, additional details on $\partial \R$ must 
be specified to allow the precise characterization of the computational 
complexity (of any algorithm developed for \opg). For this purpose, it 
is assumed that each $\partial R_i \in \partial \R$, $1 \le i \le m$, 
is a simple (i.e., non-intersecting and without holes) polygon with an 
input complexity $O(M_i)$, i.e., $\partial R_i$ has about $M_i$ vertices 
or edges. If an \opg has a single region $R$, then let $\partial R$ have 
an input complexity of $M$. Note that the algorithms developed in this 
work apply to curved boundaries equally well, provided that the curves 
have similar input complexity and are given in a format that allow the 
computation of their lengths with the same complexity. 
Alternatively, curved boundaries may be approximated to arbitrary 
precision with polygons. 


For deploying a robot to guard a $C_j$, one natural choice is to send the 
robot to a target location $t_j \in C_j$ such that $t_j$ is the centroid of $C_j$. 
Since $C_j$ is one dimensional, $t_j$ is the center (or midpoint) of $C_j$. 
After solving an \opg, there is the remaining problem of assigning the $n$ 
robots to the centers of $\C^* = \{C_j^*\}$ and actually moving the robots 
to these assigned locations. As a secondary objective, it may also be 
desirable to provide guarantees on the execution time required for 
deploying the robots to reach target guarding locations. We note that, 
the task assignment (after determining target locations) and motion 
planning component for handling robot deployment, essential for applications 
but not a key part of this work's contribution, is briefly addressed in 
Section~\ref{section:evaluation}. 
\jy{Here we are doing sequential optimization and can provide some 
guarantees. Can we do concurrent optimization here? It seems possible but 
also hard. We might be able to show that the two objectives cannot be 
simultaneously satisfied.}

With some $\C^*$ satisfying~\eqref{eq:coverage} 
and~\eqref{eq:objective}, we may further require that $len(C_j^*)$ is 
minimized for all $C_j^* \in \C^*$. This means that a gap $G \subset 
((\bigcup \partial R_i)\backslash (\bigcup P_i))$ will never be partially 
covered by some $C_j^* \in \C^*$. 
In the example from 
Fig.~\ref{fig:example-boundaries}, $G$ may be one of the gaps on 
$\partial R_1$; clearly, it is not beneficial to have some 
$C_j^*$ partially cover (i.e., intersect the interior of) one of these. 
This rather useful condition (note that this is not an assumption but 
a solution property) yields the following lemma. 

\begin{lemma}\label{l:no-partial-coverage} 
For a set of perimeters $\P = \{P_1,\ldots, P_m\}$ where $P_i \subset 
\partial R_i$ for $1 \le i \le m$, there exists an optimal cover $\C^* 
= \{C_1^*, \ldots, C_n^*\}$ such that, for any gap (or maximal connected
component) $G \subset ((\bigcup \partial R_i)\backslash (\bigcup P_i))$ 
and any $C_j^* \in \C^*$, $C_j^* \cap G = G$ or $C_j^* \cap G = 
\varnothing$. 
\end{lemma}

\begin{remark} Our definition of coverage is but one of the possible 
models of coverage. The definition restricts a robot deployed to 
$C_j, 1 \le j \le n$, to essentially {\em live} on $C_j$. The definition 
models scenarios where a guarding robot must travel along $C_j$, which 
is one-dimensional. Nevertheless, the algorithms developed for \opg have 
broader applications. For example, subroutines in our algorithms readily 
solve the problem of finding the minimum number of guards needed if each 
guard has a predetermined maximum coverage. 
\end{remark}

\section{Structural Analysis}\label{section:analysis}
In designing efficient algorithms, the solution structure of \opg induced by 
the problem formulation is first explored, starting from the case where there 
is a single region.

\subsection{Guarding a Single Region}
\noindent\textbf{Perimeter with a single connected component}. For 
guarding a single region $\R = \{R\}$, i.e., there is a single 
boundary $\partial R$ to be guarded, all $n$ robots can be directly 
allocated to $\partial R$. If the single perimeter $P \subset 
\partial R$ further has a single connected component that is either 
homeomorphic to $\mathbb S^1$ or $[0, 1]$, then each robot $j$ can 
be assigned a piece $C_j \subset P$ such that $\bigcup_{C_j \in \C} 
C_j = P$ and $len(C_j) = len(P)/n$. Clearly, such a cover $\C$ is 
also an optimal cover. 

\noindent\textbf{Perimeter with multiple maximal connected components}. 
When there are multiple maximal connected components (or segments) in a 
single perimeter $P$, things become more complex. To facilitate the 
discussion, assume here $P$ has $q$ segments $S_1, \ldots, S_q$ arranged 
in the clockwise direction (i.e., $P = S_1 \cup \ldots \cup S_q$), 
which leaves $q$ gaps $G_1, \ldots, G_q$ with $G_k$ immediately following 
$S_k$. Fig.~\ref{fig:single-perimeter} shows a perimeter with five segments
and five gaps. 
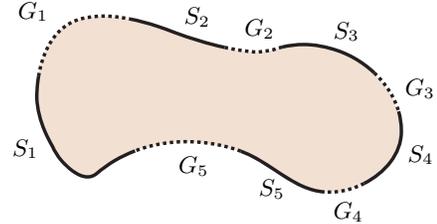
\begin{figure}[ht]
\iffull {} \fi
\begin{center}
\begin{overpic}[width={\iffull 3in \fi},tics=5]
{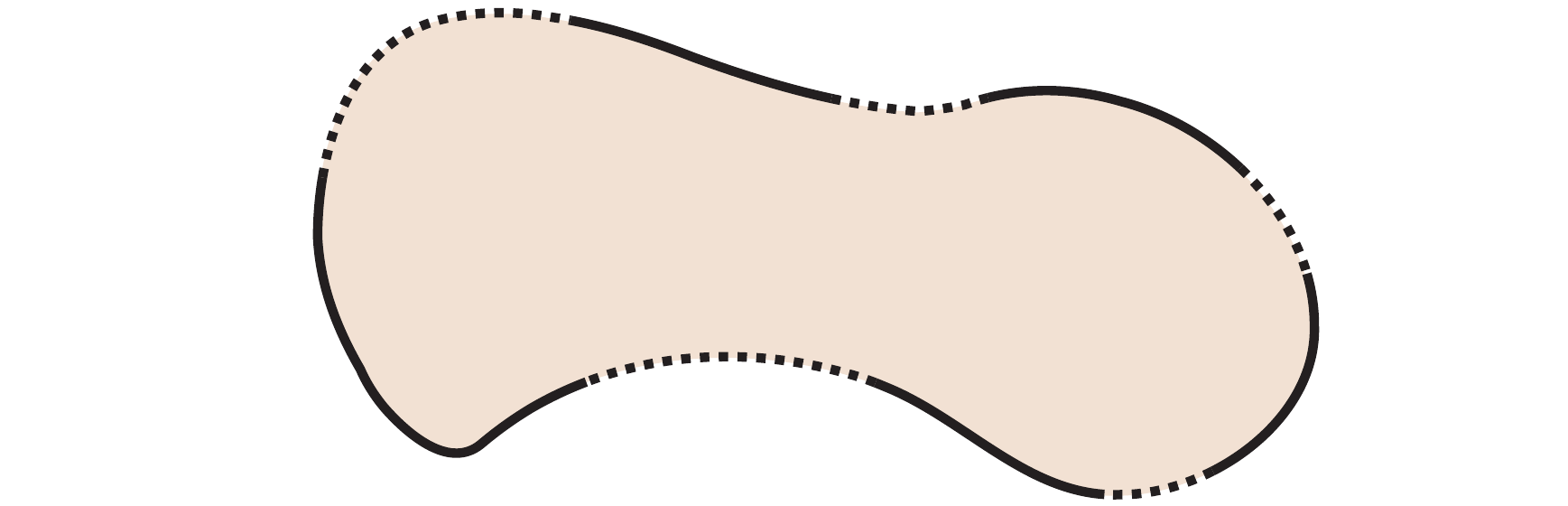}
\put(16,7){{\small $S_1$}}
\put(17,31){{\small $G_1$}}
\put(46,30){{\small $S_2$}}
\put(56.5,28){{\small $G_2$}}
\put(72,27.5){{\small $S_3$}}
\put(84.5,17.5){{\small $G_3$}}
\put(85,6){{\small $S_4$}}
\put(72,-4){{\small $G_4$}}
\put(59,0){{\small $S_5$}}
\put(45,4){{\small $G_5$}}
\end{overpic}
\end{center}
\iffull {} \fi
\caption{\label{fig:single-perimeter} A perimeter with five segments 
$S_1, \ldots, S_5$ and five {\em gaps} $G_1, \ldots, G_5$.}
\iffull {} \fi
\end{figure}

Suppose an optimal set of assignments for the $n$ robots guarding $P$ 
and  satisfying~\eqref{eq:coverage} and~\eqref{eq:objective} is $\C^* 
= \{C_j^*\}$. Let $G_{max}$ be a largest gap, i.e., $len(G_{max}) = 
\max_{1 \le k \le q}len(G_k)$. Via small perturbations to the lengths 
of $G_k$, we may also assume that $G_{max}$ is unique. On one hand, 
it must hold that $len(C_j^*) \le (len(\partial R) - len(G_{max}))/n$,
as a solution where $n$ robots evenly cover all of $\partial R$ with 
the gap $G_{max}$ excluded, satisfies the condition. On the other hand, 
$len(C_j^*) \ge (\sum_{1\le k\le q}len(S_k))/n$ always holds because 
the coverage condition requires $\sum_j C_j^* \ge \sum_{1\le k\le q}len(S_k)$. 
These yield a pair of basic upper and lower bounds for the optimal single 
robot coverage length $\ell^*$, summarized as follows. 
\jy{We may address the issue of solution stability with respect to 
perturbation (later).}

\begin{proposition}\label{p:single-bounds} Define 
\[
\ell_{min} = \frac{\sum_{1\le k\le q}len(S_k)}{n}\,\,and\,\,
\ell_{max} = \frac{len(\partial R) -  len(G_{max})}{n},
\]
it holds that 
\begin{align}\label{eq:lopt} 
\ell_{min} \le \ell^* \le \ell_{max}.
\end{align}
\end{proposition}

Though some gap, if there at least one, must be skipped by the optimal 
solution, it is not always the case that a largest gap $G_{max}$, 
even if unique, will be skipped by $\bigcup_{C_j \in \C^*} C_j^*$. 
That is, an optimal cover $C^*$ may enclose the largest gap. 
 
\begin{proposition}\label{p:max-no-exclusion}
Given a region $R$ and perimeter $P \subset \partial R$, let
$G_{max}$ be the unique longest connected component of $\partial R
\backslash P$. Let $\C^*$ be an optimal cover of $P$. Then, there 
exist \opg instances in which $G_{max} \subset C_j^*$ for some $C_j^* 
\in \C^*$. 
\end{proposition}
\begin{proof} The claim may be proved via contradiction with the example 
illustrated  in Fig.~\ref{fig:max-not-skipped} which readily generalizes. 
In the figure, there are four gaps $G_1, \ldots, G_4$, in which three 
gaps ($G_1$, $G_2$, and $G_4$) have the same length (i.e., $len(G_1)= 
len(G_2)= len(G_4)$) and are evenly spaced (i.e., $len(S_1)= len(S_2) 
= len(S_3\cup G_3\cup S_4)$). Here, $G_{max} = G_3$, which is 
$1.5$ times the length of other gaps, i.e., $len(G_3) = 
\frac{3}{2}len(G_1)$. 
\begin{figure}[ht]
\iffull \vspace*{2mm} \fi
\begin{center}
\begin{overpic}[width={\iffull 3in \fi},tics=5]
{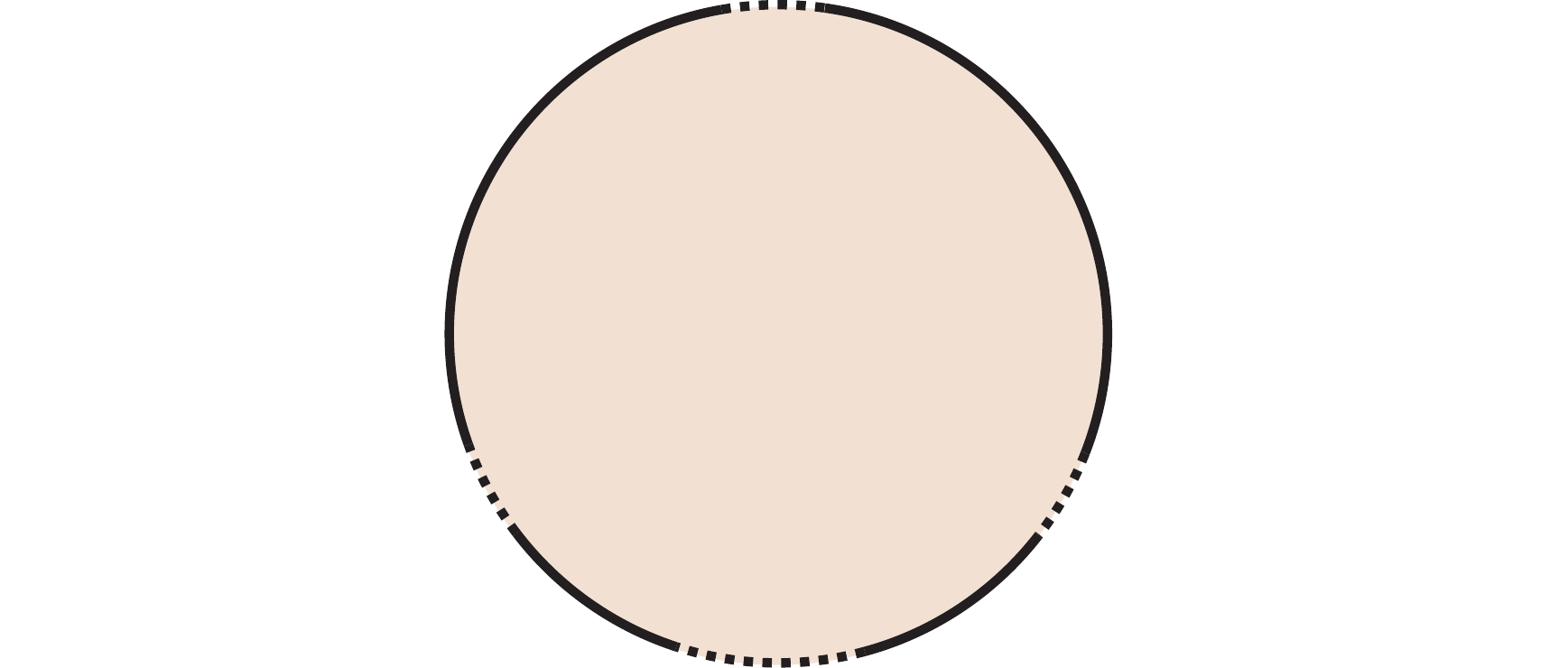}
\put(27,34){{\small $S_1$}}
\put(48,44){{\small $G_1$}}
\put(69,34){{\small $S_2$}}
\put(69,9){{\small $G_2$}}
\put(62,1){{\small $S_3$}}
\put(48,-4){{\small $G_3$}}
\put(32,2){{\small $S_4$}}
\put(26,8){{\small $G_4$}}
\end{overpic}
\end{center}
\iffull {} \fi
\caption{\label{fig:max-not-skipped} A case where the perimeter has 
four segments or maximal connected components. Three of the gaps, 
$G_1$, $G_2$, and $G_4$ are of the same length and are evenly spaced, 
$G_3$ is $0.5$ times longer.}
\iffull {} \fi
\end{figure}
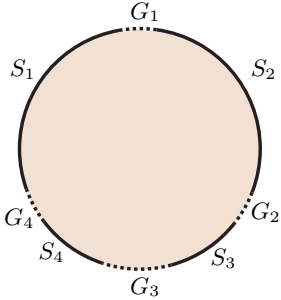

For $n = 3$ robots, the optimal cover $\C^*$ must allocate each robot 
to guard each of $S_1$, $S_2$, and $(S_3\cup G_3\cup S_4)$. Without 
loss of generality, let $C_1^* = S_1$, $C_2^* = S_2$, and $C_3^* = 
(S_3\cup G_3\cup S_4)$. This means that $G_3$ is covered by $C_3^*$ 
and not skipped by $\C^*$. In this case, $len(C_1^*) =len(C_2^*) = 
len(C_3^*) = len (S_1)$.

To see that this must be the case, suppose on the contrary that $G_3$ 
is skipped and let $\C = \{C_1, C_2, C_3\}$ be an alternative cover. 
By Lemma~\ref{l:no-partial-coverage}, an optimal cover must skip $G_3$ 
entirely. In this case, some $C_j$, say $C_1$, must have its left 
endpoint\footnote{In this paper, for a non-circular segment or gap, its 
left endpoint is defined as the {\em limit point} along the 
counterclockwise direction along the perimeter and its right endpoint 
is defined as the limit point in the clockwise direction along the 
perimeter. So, in Fig.~\ref{fig:max-not-skipped}, for $S_1$, its left 
endpoint touches $G_4$ and its right endpoint touches $G_1$.} coincide 
with the right endpoint of of $G_3$ (the point where $G_3$ meets $S_4$). 
Then $C_1$ must cover $S_4$ and $G_4$; otherwise, $C_2$ and $C_3$ must 
cover $S_1 \cup S_2 \cup S_3$, which makes $len(C_2) + len(C_3) \ge 
len(S_1 \cup S_2 \cup S_3) > 2len(S_1)$ and $\C$ a worse cover than 
$\C^*$. By symmetry, similarly, some $C_j$, say $C_3$, must have its 
right endpoint coincide with the left endpoint of $G_3$ and cover 
$S_3$ and $G_2$. However, this means that both $G_2$ and $G_4$ are 
covered by $\C$. Even if $G_1$ is skipped, this makes $len(C_1 \cup 
C_2 \cup C_3) = len(S_4 \cup G_4 \cup S_1 \cup S_2 \cup G_2 \cup S_3) 
> len (S_1 \cup S_2 \cup S_3\cup G_3\cup S_4) = 3len(S_1)$, again making
$\C$ sub-optimal. By the pigeonhole principle, at least one of the $C_1$, 
$C_2$, or $C_3$ must be longer than $len(S_1)$. Therefore, skipping 
$G_{max} = G_3$ in this case leads to a sub-optimal cover. The  
optimal cover with $n = 3$ is to have $C^* = \{S_1, S_2, 
(S_3\cup G_3\cup S_4)\}$. 
\end{proof}

Proposition~\ref{p:max-no-exclusion} implies that in allocating robots 
to guard a perimeter $P \subset \partial R$, an algorithm cannot simply 
start by excluding the longest component from $\partial R 
\backslash P$ and then the next largest, and so on. This makes solving
\opg more challenging. Referring back to Fig.~\ref{fig:example}, if the 
top gap is skipped by the cover, then the three robots on the right side 
of the perimeter (two orange and one green) need to cover the part of the 
perimeter between the two hexagons. This will cause $\ell^*$ to increase. 

On the other hand, for an optimal cover $\C^* = \{C_1^*, \ldots, C_n^*\}$ 
of $P$, some $C_j^*\in \C^*$ must have at least one of its endpoint 
aligned with an endpoint of a component $S_k$ of $P$ (assuming that $P 
\subsetneq \partial R$). 

\begin{proposition}\label{p:endpoint-alignment}
For an optimal cover $\C^* = \{C_1^*,\ldots, C_n^*\}$ of a perimeter 
$P = S_1 \cup \ldots \cup S_q \subset \partial R = S_1 \cup G_1 
\cup $ $\ldots \cup S_q \cup G_q$, for some $S_i \subset P$ and 
$C_j^* \in C^*$, their right (or left) endpoints must coincide. 
\end{proposition}
\begin{proof}
By Lemma~\ref{l:no-partial-coverage}, for any $G_k \subset \partial R 
\backslash P$, and $C_j^* \in \C^*$, $G_k \cap C_j^* = G_k$ or 
$G_k \cap C_j^* = \varnothing$. Since at least one $G_k$, $1 \le k \le q$, 
must be skipped by $C_1^* \cup \ldots C_n^*$, some $C_j^*$, $1 \le j \le n$ 
must have its right endpoint aligned with the right endpoint of $S_k$, 
which is on the left of $G_k$. Following the same argument, some 
$C_{j'}^*$ and $S_{k'}$ must have the same left endpoints. 
\end{proof}

Proposition~\ref{p:endpoint-alignment} suggests that we may attempt to cover 
a perimeter $P$ starting from an endpoint of $S_1$, $S_2$, and so on. 
Indeed, as we will show in Section~\ref{section:algorithm}, an efficient 
algorithm can be designed exploiting this important fact. 
\jy{There is something interesting here: it seems that the algorithm not only
finds a solution satisfying~\eqref{eq:objective}, but also ensures a local 
optimal solution that minimizes the max coverage for all other robots that 
do not cover $\ell^*$. This may be interesting to discuss further.}

\subsection{Guarding Multiple Regions}
In a multiple region setup, there is one additional level of complexity:
the number of robots that will be assigned to an individual region is 
no longer fixed. This introduces another set of variables $n_1, \ldots, 
n_m$ with $n_1 + \ldots + n_m = n$, and $n_i$, $1 \le i \le m$ being the 
number of robots allocated to guard $\partial R_i$. For a fixed $n_i$, 
the results derived for a single region, i.e., 
Propositions~\ref{p:single-bounds}--\ref{p:endpoint-alignment} continue 
to hold.

\section{Efficient Algorithms for Perimeter Guarding}\label{section:algorithm}
\def\algoMRSimple{{\sc MultiRegionSingleComp}\xspace}
\def\algoSRG{{\sc SingleRegionMultiComp}\xspace}
\def\algoMRG{{\sc MultiRegionMultiComp}\xspace}
\def\isLFeasible{{\sc IsFeasible}}
\def\isLFeasibleByTilingPartial{{\sc IsTilingFeasiblePartial}}
\def\isLFeasibleByTilingFull{{\sc IsTilingFeasibleFull}}
In presenting algorithms for \opg, we begin with the case 
where each perimeter $P_i \in \P$ has a single connected component 
(i.e., $P_i$ is homeomorphic to $\mathbb S^1$ or $[0, 1]$). Then, 
we work on the general single region case where the only perimeter 
is composed of $q > 1$ connected components, before moving to the 
most general multiple regions case.

\subsection{Perimeters Containing Single Components}
When there is a single perimeter $P$, the solution is straightforward 
with $\ell^* = len(P)/n$. With $\ell^*$ determined, $\C^*$ is also readily computed. 

In the case where there are $m > 1$ regions, let the optimal 
distribution of the $n$ robots among the $m$ regions be given by 
$n_1^*, \ldots, n_m^*$. For a given region $R_i$, the $n_i^*$ robots 
must each guard a length $\ell_i = len(P_i)/n_i^*$. At this point, we
observe that for at least one region, say $R_i$, 
the corresponding $\ell_i$ must be maximal, i.e., $\ell_i = \ell^*$. 
The observation directly leads to a naive strategy for finding $\ell^*$: 
for each $R_i$, one may simply try all possible $1 \le n_i \le n$ and 
find the maximum $len(P_i)/n_i$ that is feasible, i.e., $n - n_i$ robots
can cover all other $R_{i'}$, $i' \ne i$, with each robot covering no 
more than $len(P_i)/n_i$. Denoting this candidate cover length 
$len(P_i)/n_i$ as $\ell_i^c$ and the corresponding $n_i$ as $n_i^c$, the 
smallest $\ell_i^c$ overall $1 \le i \le m$ is then $\ell^*$.

The basic strategy mentioned above works and runs in polynomial time. 
It is possible to carry out the computation much more efficiently if 
the longest $P_i$ is examined first. Without loss of generality, 
assume that $P_1$ is the longest perimeter, i.e., $len(P_1)
\ge len(P_{i})$ for all $1 \le i \le m$. Recall that $n_1^c$ is the 
number of robots allocated to $P_1$ that yields $\ell_1^c$, it must hold 
that 
\begin{align}\label{eq:l1}
\frac{len(P_1)}{n_1^c + 1} < \ell^* \le \frac{len(P_1)}{n_1^c} = \ell_1^c .
\end{align}
For an arbitrary $P_i$, simple manipulating of~\eqref{eq:l1} yields
\begin{align}\label{eq:li}
\frac{len(P_i)}{(n_1^c + 1)\frac{len(P_i)}{len(P_1)}} < \ell^* \le 
\frac{len(P_i)}{n_1^c\frac{len(P_i)}{len(P_1)}}.
\end{align}
This means that we only need to consider
$
n_i^c \in 
\big[\lceil n_1^c\frac{len(P_i)}{len(P_1)} \rceil, 
\lfloor (n_1^c + 1)\frac{len(P_i)}{len(P_1)}\rfloor].
$
Moreover, since $P_1$ is the longest perimeter, $\frac{len(P_i)}{len(P_1)} 
\le 1$. Therefore, the difference between the two denominators 
of~\eqref{eq:li} is no more than $1$, i.e., 
\[
(n_1^c + 1)\frac{len(P_i)}{len(P_1)} - n_1^c\frac{len(P_i)}{len(P_1)} \le 1. 
\]
When $len(P_i) \ne len(P_1)$, $(n_1^c + 1)\frac{len(P_i)}{len(P_1)} 
- n_1^c\frac{len(P_i)}{len(P_1)} < 1$ and there are two possibilities. One 
of these is 
$
\lceil n_1^c\frac{len(P_i)}{len(P_1)} \rceil =
\lfloor (n_1^c + 1)\frac{len(P_i)}{len(P_1)}\rfloor,
$
which leaves a single possible candidate for $n_i^c$. The other 
possibility 
is 
$
\lceil n_1^c\frac{len(P_i)}{len(P_1)} \rceil =
\lfloor (n_1^c + 1)\frac{len(P_i)}{len(P_1)}\rfloor + 1,
$
in which case there is actually no valid candidate for $n_i^c$. 
That is, after computing $n_1^c$ and $\ell_1^c$, if $len(P_i) = len(P_1)$ 
then no computation is needed for $P_i$. If $len(P_i) < len(P_1)$ then we 
only need to check at most one candidate for $n_i^c$. 


Additional heuristics can be applied to reduce the required computation. 
First, in finding $n_1^c$, we may use bisection (binary search) over 
$[1, m]$ since if a given $n_1$ is infeasible, any $n_1' > n_1$ cannot 
be feasible either because $len(P_1)/n_1 < len(P_1)/n_1'$. Second, let 
$\ell = (\sum_{1\le i\le m}len(P_i))/n$, it holds that $\ell_i^c \ge 
\ell^* \ge \ell$. This means that for each $1 \le i \le m$, it is not 
necessary to try any $n_i > \lfloor \frac{len(P_i)}{\ell} \rfloor$. 
Third, if a candidate $\ell_i^c$ is at any time larger than the current 
candidate for $\ell^*$, that $i$ does not need to be checked further. 
We only use the first and the third in our implementation since the 
second does not help much once the bisection step is applied. The 
pseudo code is outlined in Algorithm~\ref{algo:MRS}. Note that we 
assume the problem instance is feasible ($n \ge m$), which is easy to check. 

It is straightforward to verify that Algorithm~\ref{algo:MRS} runs in 
time $O(m\log n + m^2)$. The $O(m\log n)$ comes from the $\mathbf{while}$ 
loop, which calls the function \isLFeasible($\ell_i^c$, $n_i^c$, $i$) 
$\log n$ times. The function checks whether the current $\ell_i^c$ is 
feasible for perimeters other than $P_i$ (note that it is assumed that 
\isLFeasible($\cdot$) has access to the input to Algorithm~\ref{algo:MRS} 
as well). This is done by computing for $i' \ne i$, $n_{i'} = \lceil 
len(P_{i'})/\ell_i^c \rceil$ and checking whether $\sum_{{i'} \ne i}n_{i'} 
\le n - n_i^c$. The $O(m^2)$ term comes from the $\mathbf{for}$ loop. 
The running time of Algorithm~\ref{algo:MRS} may be further reduced by 
noting that the $\mathbf{for}$ loop examines $(m-1)$ candidate 
$\ell_i^c$. These $\ell_i^c$ can be first computed and sorted, on which 
bisection can be applied. This drops the main running time to 
$O(m(\log n + \log m))$. This second bisection is not reflected in 
Algorithm~\ref{algo:MRS} to keep the logic and notation more 
straightforward. If we also consider input complexity, an additional 
$O(\sum_{1\le i \le m} M_i)$ is needed to compute $len(P_i)$ from the raw polygonal 
input and an additional $O(n)$ time is needed for generating the actual
locations for the $n$ robots. The total complexity is then $O(m(\log n + 
\log m) + \sum_{1\le i \le m} M_i + n)$.
\iffull {} \fi
\begin{algorithm}
\begin{small}
    \SetKwInOut{Input}{Input}
    \SetKwInOut{Output}{Output}
    \SetKwComment{Comment}{\%}{}
    \Input{$P_1, \ldots, P_m$: each $P_i$ a polygon or p-chain; 
		assume that $P_1$ is a longest perimeter \\
		$n$: the number of robots}
    \Output{$\ell^*, i^*$: the optimal coverage and the $i$ realizing it}
\vspace{0.025in}

$\ell^* \leftarrow \infty$; $i^* \leftarrow 1$;
\vspace{0.025in}

\Comment{\footnotesize Compute $n_1^c$ and initial $\ell^*$.}

$n_1^{min} \leftarrow 1$; $n_1^{max} \leftarrow n$; $n_1^c \leftarrow 1$;
\vspace{0.025in}

\While{$n_1^{min} \ne n_1^{max}$}{
\vspace{0.025in}
$n_1 \leftarrow \lceil \frac{n_1^{min} + n_1^{max}}{2} \rceil$;
$\ell_1 \leftarrow \frac{len(P_1)}{n_1}$;

\vspace{0.025in}
\uIf{\isLFeasible($\ell_1$, $n_1$, $1$)}{
	\vspace{0.025in}
	$\ell^* \leftarrow \ell_1$; 
	$n_1^c \leftarrow n_1$;
	$n_1^{min} \leftarrow n_1$; 
}
\Else{
	$n_1^{max} \leftarrow n_1 - 1$; 
}
}

\vspace{0.05in}
\Comment{\footnotesize Optimize $\ell^*$ over all $2 \le i \le m$.}

\For{$i \in \{2,\dots, m\}$ }{\label{algo:mrs:for}
	$n_i^- = \lceil \frac{n_1^clen(P_i)}{len(P_1)} \rceil$;
	$n_i^+ = \lfloor \frac{(n_1^c +1)len(P_i)}{len(P_1)} \rfloor$; 
	$\ell_i \leftarrow \frac{len(P_i)}{n_i^+}$;
				
		\If{$n_i^- == n_i^+$ and \isLFeasible($\ell_i$, $n_i^+$, $i$) and $\ell_i < \ell^*$}{
			\vspace{0.025in}
			$\ell^* \leftarrow \ell_i$; $i^* \leftarrow i$;
		}
}
\Return{$\ell^*$, $i^*$}
\caption{\algoMRSimple} \label{algo:MRS}
\end{small}
\end{algorithm}
\iffull {} \fi

\subsection{Single Perimeter Containing Multiple Components}
\noindent\textbf{Additional structural analysis}. In computing $\ell^*$ for a 
single perimeter $P$ with multiple connected 
components, assume that $P$ is composed of $q$ maximal connected 
components $S_1, \ldots, S_q$ (e.g., Fig.~\ref{fig:single-perimeter}),
leaving $G_1, \ldots, G_q$ as the gaps on $\partial R$. Given an 
optimal cover $\C^* = \{C_1^*, \ldots, C_n^*\}$, by 
Proposition~\ref{p:endpoint-alignment}, we may assume that the left
endpoint of some $C_j^*$, $1 \le j \le n$ coincides with the left 
endpoint of some $S_k$, $1 \le k \le q$. We then look at the right 
endpoint of $C_j^*$. If it does not coincide with the right endpoint 
of some $S_{k'}$ ($k$ and $k'$ may or may not be the same), it must 
coincide with the left endpoint of $C_{j+1}^*$. Continuing like this, 
eventually we will hit some $C_{j'}^*$ where the right endpoint of 
$C_{j'}^*$ coincides with the right endpoint of some $S_{k'}$. 
Within a partitioned subset $C_j^*, \ldots, C_{j'}^*$, the maximal 
coverage of each robot is minimized when $len(C_j^*) = \ldots = 
len(C_{j'}^*)$. Because $\ell^* = len(C_j^*)$ for some $1 \le j \le n$, 
at least one of the subsets must have all robots cover exactly a 
length of $\ell^*$. These two key structural observations are 
summarized as follows. 

\iffull {} \fi
\begin{theorem}\label{t:optimal-partition}
Let $\C^* = \{C_1^*, \ldots, C_n^*\}$ be a solution to an \opg instance
with a single perimeter $P = S_1\cup \ldots\cup S_q$ and gaps $G_1, \ldots, 
G_q$. Then, $\C^*$ may be partitioned into disjoint subsets with
the following properties
\begin{enumerate}
\item the union of the individual elements from any subset forms a 
continuous p-chain, 
\item the left endpoint of such a union coincides with the left 
endpoint of some $S_k$, $1 \le k \le q$, 
\item the right endpoint of such a union coincides with the right 
endpoint of some $S_{k'}$, $1 \le k' \le q$,  and
\item the respective unions of elements from any two subsets are 
disjoint, i.e., they are separated by at least one gap. 
\end{enumerate}
Moreover, for at least one such subset, $\{C_j^*, \ldots, C_{j'}^*\}$, 
it holds that $\ell^* = len(C_j^*) = \ldots = len(C_{j'}^*)$. 
\iffull {} \fi
\end{theorem}
In the example from Fig.~\ref{fig:example}, $\C^*$ is partitioned into 
two subsets satisfying the conditions stated in 
Theorem~\ref{t:optimal-partition}. 

\noindent\textbf{A baseline algorithm}.
The theorem provides a way for 
computing $\ell^*$. For fixed $1 \le k, k' \le q$, denote the part of 
$\partial R$ between $S_k$ and $S_{k'}$ following a clockwise direction 
(with $S_k$ and $S_{k'}$ included) as $S_{k-k'}$.
Theorem~\ref{t:optimal-partition} says that for some $k, k'$, $len(S_{k-k'})
= n_{k-k'}^*\ell^*$ for some integer $n_{k-k'}^* \in [1, n]$. We may 
find $k, k'$, and $n_{k-k'}^*$, $\ell^*$ by exhaustively going through all 
possible $k, k'$, and $n_{k-k'}^c$ (as a candidate of $n_{k-k'}^*$). For 
each combination of $k, k'$ and $n_{k-k'}^c$, we can compute a 
\begin{align}\label{eq:lkkc}
\iffull {} \fi
\ell_{k-k'}^c = \frac{len(S_{k-k'})}{n_{k-k'}^c}
\iffull {} \fi
\end{align}
and check $\ell_{k-k'}^c$'s feasibility. The largest feasible 
$\ell_{k-k'}^c$ is $\ell^*$. 

\vspace{1mm}
\noindent\underline{Partial feasibility check}: 
For checking feasibility of a particular $\ell_{k-k'}^c$, i.e., whether 
$\ell_{k-k'}^c$ is long enough for the rest of the robots to cover the 
rest of the perimeter, we simply {\em tile} $(n - n_{k-k'}^c)$ copies 
$\ell_{k-k'}^c$ over the rest of the perimeter, starting from $S_{(k' 
\mod q) + 1}$. As an example, see Fig.~\ref{fig:tiling} where $n = 6$ 
robots are to cover the perimeter (in red, with five components $S_1, 
\ldots, S_5$). Suppose that the algorithm is currently working with 
$S_{1-2}$ (i.e., $k=1$ and $k' = 2$). If  $n_{1-2}^c = 2$, then 
$\ell_{1-2}^c = len(S_{1-2})/2$. Each of the five green line segments 
$C_1, \ldots, C_5$  in the figure has this length. As visualized in the 
figure, it is possible to cover $P\backslash S_{1-2}$ with three more 
robots, which is no more than $n - n_{1-2}^c = 4$. Therefore, this 
$\ell_{1-2}^c$ is feasible; note that it is not necessary to exhaust 
all $n = 6$ robots. In the figure, $C_3$ covers the entire $S_3$ and 
$G_3$, as well as part of $S_4$. The rest of $S_4$ is covered by $C_4$. 
As $C_4$ is tiled, it ends in the middle of $G_4$, so $C_5$ starts at 
the beginning of $S_5$. 
On the other hand, if $n_{1-2}^c = 3$, the resulting $\ell_{1-2}^c$ 
(each of the orange line segments has this length) is infeasible as 
$S_5$ is now left uncovered.
\begin{figure}[ht]
	\iffull {} \fi
	\begin{center}
		\begin{overpic}[width={\ifoc 4in \else 3in \fi},tics=5]{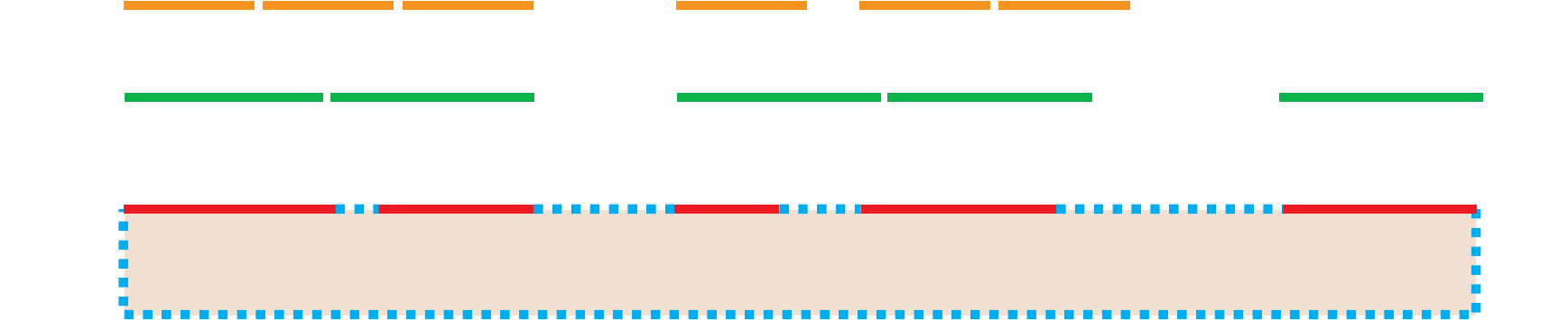}
			\put(13,2.7){{\small $S_1$}}
			\put(27,2.7){{\small $S_2$}}
			\put(45,2.7){{\small $S_3$}}
			\put(60,2.7){{\small $S_4$}}
			\put(86.5,2.7){{\small $S_5$}}
			\put(12.5,10){{\small $C_1$}}
			\put(25,10){{\small $C_2$}}
			\put(48,10){{\small $C_3$}}
			\put(61,10){{\small $C_4$}}
			\put(86.5,10){{\small $C_5$}}
		\end{overpic}
	\end{center}
	\iffull {} \fi
	\caption{\label{fig:tiling}  An illustration of the feasibility check of 
		$\ell_{1-2}^c$. The single rectangular region and the perimeter (five red 
		segments $S_1$--$S_5$) are shown at the bottom. The orange and green line 
		segments show two potential	covers.}
	\iffull {} \fi
\end{figure}
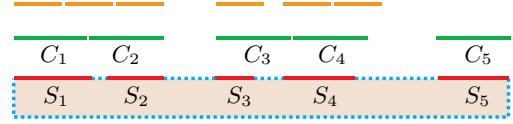

The tiling-based feasibility check takes $O(q)$ time as there are at 
most $q$ segments to tile; it takes constant time to tile each 
using a given length. Let us denote this feasibility check 
\isLFeasibleByTilingPartial($k$, $k'$, $n_{k-k'}^c$), we have obtained 
an algorithm that runs in $O(nq^3)$ times since it needs to go through 
all $1 \le k \le q$, $1 \le k' \le q$, and $1 \le n_{k-k'}^c \le n$.
For each combination of $k, k'$, and $n_{k-k'}^c$, it makes a call to 
\isLFeasibleByTilingPartial($\cdot$). While a $O(nq^3)$ running time 
is not bad, we can do significantly better. 

\vspace{1mm}
\noindent\textbf{A much faster algorithm}.
In the baseline algorithm, for each $k-k'$ combination, up to $n$ 
candidate $n_{k-k'}^c$ may be attempted. To gain speedups, the first 
phase of the improved algorithm reduces the range of $\ell^*$ to limit 
the choice of $n_{k-k'}^c$. For the faster algorithm, a new feasibility 
checking routine is needed. 

\vspace{1mm}
\noindent\underline{Full feasibility check}: 
We introduce a feasibility check given only a length $\ell$. That is, a 
check is done to see whether $n$ robots are sufficient for covering $P$ 
without any covering more than length $\ell$. This feasibility check is 
performed in a way similar to \isLFeasibleByTilingPartial($\cdot$) but 
now $k$ and $k'$ are not specified. We instead try all $S_k$, $1 \le k 
\le q$ as the possible starting segment for the tiling. Let us denote 
this procedure \isLFeasibleByTilingFull($\ell$), which runs in $O(q^2)$.

\vspace{1mm}
\noindent\underline{Using bisection to limit the search range for 
$\ell^*$}: Starting from the initial bounds for $\ell^*$ given in 
Proposition~\ref{p:single-bounds} and with 
\isLFeasibleByTilingFull($\ell$), we can narrow the bound to be 
arbitrarily small, using bisection, since $\ell^*$ is the minimum 
feasible $\ell$. To do this, we start with $\ell$ as the middle 
point of initial lower bound $\ell_{min}$ and upper bound 
$\ell_{max}$, and run \isLFeasibleByTilingFull($\ell$). If $\ell$ 
is feasible, the upper bound is lowered to $\ell$. Otherwise, the 
lower bound is raised to $\ell$. In doing this, our goal in the 
first phase of the faster algorithm is to reduce the range for 
$\ell^*$ so that there is at most a single choice for $n_{k-k'}^c$,
regardless of the values of $k$ and $k'$. 
\iffull
The stopping criteria for the bisection is given as follows.
\fi

\iffull {} \fi
\begin{proposition}\label{p:bisect-stop}
Assume that the bisection search stops with lower and upper bound 
being $\ell_{min}^f$ and $\ell_{max}^f$. If 
\begin{align}\label{eq:lmmd2}
\ell_{max}^f- \ell_{min}^f 
< \frac{\Big[\sum_{1\le k\le q}len(S_k)\Big]^2}{n^2len(\partial R)}, 
\end{align}
then there is at most a single choice for $n_{k-k'}^c$ for all $k, k'$.
\end{proposition}
\iffull
\begin{proof}
Assume that when the bisection ends, the lower and upper bounds are 
$\ell_{min}^f$ and $\ell_{max}^f$. Then, for some $\ell_{k-k'}^c$
(see~\eqref{eq:lkkc}), it holds that
\[
\ell_{min}^f \le \ell_{k-k'}^c = \frac{len(S_{k-k'})}{n_{k-k'}^c} \le \ell_{max}^f.
\]
Rearranging the above yields
\[
\frac{len(S_{k-k'})}{\ell_{max}^f} \le n_{k-k'}^c \le \frac{len(S_{k-k'})}{\ell_{min}^f}.
\]
To reduce the number of possible $n_{k-k'}^c$ to at most $1$, we want 
\[
\frac{len(S_{k-k'})}{\ell_{min}^f} - \frac{len(S_{k-k'})}{\ell_{max}^f} < 1, 
\]
which is the same as requiring 
\begin{align}\label{eq:lmmd}
\ell_{max}^f- \ell_{min}^f <
\frac{\ell_{max}^f\ell_{min}^f}{len(S_{k-k'})}. 
\end{align}
Noting that $\ell_{min}^f, \ell_{max}^f \ge  \ell_{min}$ and $len(S_{k-k'}) < 
len (\partial R)$, Equation~\eqref{eq:lmmd} is ensured by 
\begin{align}
\ell_{max}^f- \ell_{min}^f 
< \frac{\ell_{min}^2}{len(\partial R)} 
= \frac{\Big[\sum_{1\le k\le q}len(S_k)\Big]^2}{n^2len(\partial R)} 
\end{align}
\end{proof}
\fi

\noindent\underline{Finding $\ell^*$}: 
Equation~\eqref{eq:lmmd2} gives the stopping criteria used for refining 
the bounds for $\ell^*$. After completing the first phase, the algorithm
moves to the second phase of actually pinning down $\ell^*$. In this phase, 
instead of checking $\ell_{k-k'}^c$ one by one, we collect $\ell_{k-k'}^c$ for
all possible combinations of $k, k'$. Because the first phase already ensures
for each $k, k'$ combination there is at most one pair of $n_{k-k'}^c$ and 
$\ell_{k-k'}^c$, there are 
at most $q^2$ total candidates. After all candidates are collected, they are 
sorted and another bisection is performed over these sorted candidates. 
Feasibility check is done using \isLFeasibleByTilingPartial($\cdot$). The 
complete algorithm is given in Algorithm~\ref{algo:SRG}. Note that 
$\ell^{min}$ and $\ell^{max}$, which change as the algorithm runs, are not 
the same as the fixed $\ell_{min}$ and $\ell_{max}$ from 
Proposition~\ref{p:single-bounds}. 

In terms of running time, the first $\mathbf{while}$ loop starts with 
$\ell^{max} - \ell^{min} = \frac{len(\partial R) -  len(G_{max})}{n} -
\frac{\sum_{1\le k\le q}len(S_k)}{n} \le \frac{len(\partial R)}{n}$ 
and 
stops when $\ell^{max} - \ell^{min} \le 
\frac{[\sum_{1\le k\le q}len(S_k)]^2}{n^2len(\partial R)}$. Therefore, 
the bisection is executed 
$\log \frac{n[len(\partial R)]^2}{[\sum_{1\le k\le q}len(S_k)]^2}$ times, 
which by
the assumption that $len(\partial R)$ is a polynomial factor over 
$\sum_{1\le k\le q}len(S_k)$, is $O(\log (n + q))$. Since each feasibility 
check takes $O(q^2)$ time, the first $\mathbf{while}$ loop takes 
$O(q^2\log(n + q))$ time. The $\mathbf{for}$ loops work with a total of 
$O(q^2)$ candidates and must sort them, taking time $O(q^2 \log q^2) = 
O(q^2 \log q)$. Then, the second $\mathbf{while}$ loop bisects $O(q^2)$ 
candidates and calls \isLFeasibleByTilingPartial($\cdot$) for each check, 
taking time $O(q\log q^2) = O(q\log q)$. The total running time of 
Algorithm~\ref{algo:SRG} is then $O(q^2\log (n + q) + M + n)$.\footnote{We note that 
the assumption that $len(\partial R)$ is a polynomial factor over 
$\sum_{1\le k\le q}len(S_k)$ is not necessary. However, the corresponding 
analysis becomes much more involved without it. Since the assumption makes 
practical sense and also due to space limit, the more general result is 
omitted from the current paper. Many additional interesting but 
non-essential details, including this one, will be included in an extended 
version.}

\begin{algorithm}
	\begin{small}
		\SetKwInOut{Input}{Input}
		\SetKwInOut{Output}{Output}
		\SetKwComment{Comment}{\%}{}
		\Input{$\partial R = S_1 \cup G_1 \cup \ldots \cup S_q \cup G_q$: a 
			single boundary with the perimeter $P = S_1 \cup \ldots \cup S_q$.\\
			$n$: the number of robots}
		\Output{$\ell^*, k^*, k'^{*}$: the optimal coverage and the 
			pair of $k$ and $k'$ that realize the optimal coverage}
		\vspace{0.025in}
		
		\Comment{\footnotesize Phase one: narrow the range of $\ell^*$.}
		$\ell^{min} \leftarrow \frac{\sum_{1\le k\le q}len(S_k)}{n}$, 
		$\ell^{max} \leftarrow  \frac{len(\partial R) -  len(G_{max})}{n}$;
		\vspace{0.025in}
		
		\While{$\ell^{max}-\ell^{min} > \frac{[\sum_{1\le k\le q}len(S_k)]^2}{n^2len(\partial R)}$}{
			$\ell \leftarrow \frac{\ell^{max}+\ell^{min}}{2}$;

			\vspace{0.025in}
			(\,\isLFeasibleByTilingFull($\ell$)? $\ell^{max} \leftarrow \ell$ : 
			$\ell^{min} \leftarrow \ell$\,);

		}
		
		\vspace{0.025in}
		\Comment{\footnotesize Phase two: pin down $\ell^*$.}
		$sm \leftarrow []$; \Comment{$sm$ is a sorted map.}

		\For{$k, k' \in \{1,\dots, q\}$ }{\label{algo:srg:for}
			\vspace{0.025in}
			
			$n_{k-k'}^{max} \leftarrow \lfloor \frac{len(S_{k-k'})}{\ell^{min}} \rfloor$;
			$n_{k-k'}^{min} \leftarrow \lceil \frac{len(S_{k-k'})}{\ell^{max}} \rceil$;
			
			\vspace{0.025in}
			
			\For{$n_{k-k'}^c \in \{n_{k-k'}^{min}, \ldots, n_{k-k'}^{max}\}$}{\label{algo:srg:for2}
				$sm$.put($\frac{len(S_{k-k'})}{n_{k-k'}^c}$, $(n_{k-k'}^c, \frac{len(S_{k-k'})}{n_{k-k'}^c}, k, k')$);
			}
		}
	
		\vspace{0.025in}
		$\ell^* \leftarrow \infty$; $k^* \leftarrow 0$; $k'^* \leftarrow 0$;
		
		\vspace{0.025in}
		\While{$sm.\mathrm{size()}$ $> 1$}{
			\Comment{\footnotesize Extract the element from sm in the middle.}
			$(n^c, \ell^c, k, k') \leftarrow sm$.middleValue(); 
	
			\vspace{0.025in}
			\uIf{\isLFeasibleByTilingPartial($k, k', n^c$)}{
				$\ell^* \leftarrow \ell^c$; $k^* \leftarrow k$; $k'^* \leftarrow k'$;
				
				$sm \leftarrow sm.$range($sm.$minKey(), $\ell^c$);
			}
			\Else{
				$sm \leftarrow sm.$removeRange($sm.$minKey(), $\ell^c$);
			}
		}		
		
		\Return{$\ell^*$, $k^*$, $k'^*$}
		\caption{\algoSRG} \label{algo:SRG}
	\end{small}
\end{algorithm}

\vspace*{-2mm}
\subsection{Multiple Perimeters Containing Multiple Components}
\vspace*{-1mm}
The algorithm for the multiple perimeter case is a direct 
generalization Algorithm~\ref{algo:SRG}. To facilitate the description, 
let the perimeter $P_i$, $1 \le i \le m$, contain $q_i$ maximal connected 
components, i.e., $P_i = S_{i,1} \cup\ldots \cup S_{i,q_i}$ and the 
boundary $\partial R_i = S_{i,1} \cup G_{i,1} \cup \ldots \cup S_{i,q_i} 
\cup G_{i,q_i}$. We extend the definition of $S_{k-k'}$ for a single 
perimeter to $S_{i,k-k'}$ for multiple perimeters. By a straightforward 
generalization of Theorem~\ref{t:optimal-partition} to multiple perimeters,
for an \opg instance, the length of some $S_{i,k-k'}$ must be an integer 
multiple of $\ell^*$. Similar to the single perimeter case, we can try 
all $S_{i,k-k'}$ and for each try all possible $1 \le n_{i,k-k'}^c \le n$. 
This gives us $\ell_{i,k-k'}^c = \frac{len(S_{i,k-k'})}{n_{i,k-k'}^c}$ as 
candidates for $\ell^*$; there are $n(\sum_{1\le i \le m} q_i^2)$ such 
candidates. For checking the feasibility of $\ell_{i,k-k'}^c$, we may use 
\isLFeasibleByTilingPartial($\cdot$) for the rest of $P_i$ (taking $O(q_i)$ 
time) and \isLFeasibleByTilingFull($\cdot$) for all $1 \le i' \le m$ and $i' 
\ne i$ (taking $O(\sum_{1 \le i' \le m, i' \ne i} q_{i'}^2)$ time). 
This yields a baseline algorithm that runs in 
$O(n(\sum_{1\le i \le m} q_i^2)^2)$ time. 

From here, speedups can be obtained as in the single perimeter case using 
the same reasoning. This yields a two-phase algorithm, which we call 
\algoMRG, that runs in $O((\sum_{1\le i \le m} q_i^2) \log(n + 
\sum_{1\le i \le m} q_i) + \sum_{1\le i \le m} M_i + n)$.

\section{Performance Evaluation and Applications}\label{section:evaluation}
Our evaluation first verifies the algorithms' running time matches 
the claimed bounds. Then, two practical scenarios are illustrated 
to show how \opg may be adapted to applications. 

\vspace*{-1mm}
\subsection{Algorithm Performance}
\vspace*{-1mm}
In the performance results presented here, a data point is the 
average from $10$ randomly generated \opg instances. All algorithms 
are implemented in Python 2.7, and all experiments are executed on 
an Intel\textsuperscript{\textregistered} Xeon\textsuperscript{\textregistered} 
CPU at 3.0GHz. 

For the case of $m$ perimeters each containing a single segment, for 
each $1 \le i \le m$, we set $len(\partial R_i) = 1$ and let $len(P_i)$ 
be uniformly distributed in $(0, 1]$. Fig.~\ref{fig:mpsc-example} shows 
the result for an example with $m = 10$ and $n = 30$. For various values 
of $m, n$, the running time of \algoMRSimple is summarized in 
Table~\ref{eval:mpsc}, which scales very well with $m$ and $n$ (note that 
the $n \le m$ case does not make much sense here). 

\begin{figure}[ht!]
    \iffull {} \fi
    \centering
		\iffull
		\includegraphics[keepaspectratio, scale=0.4]{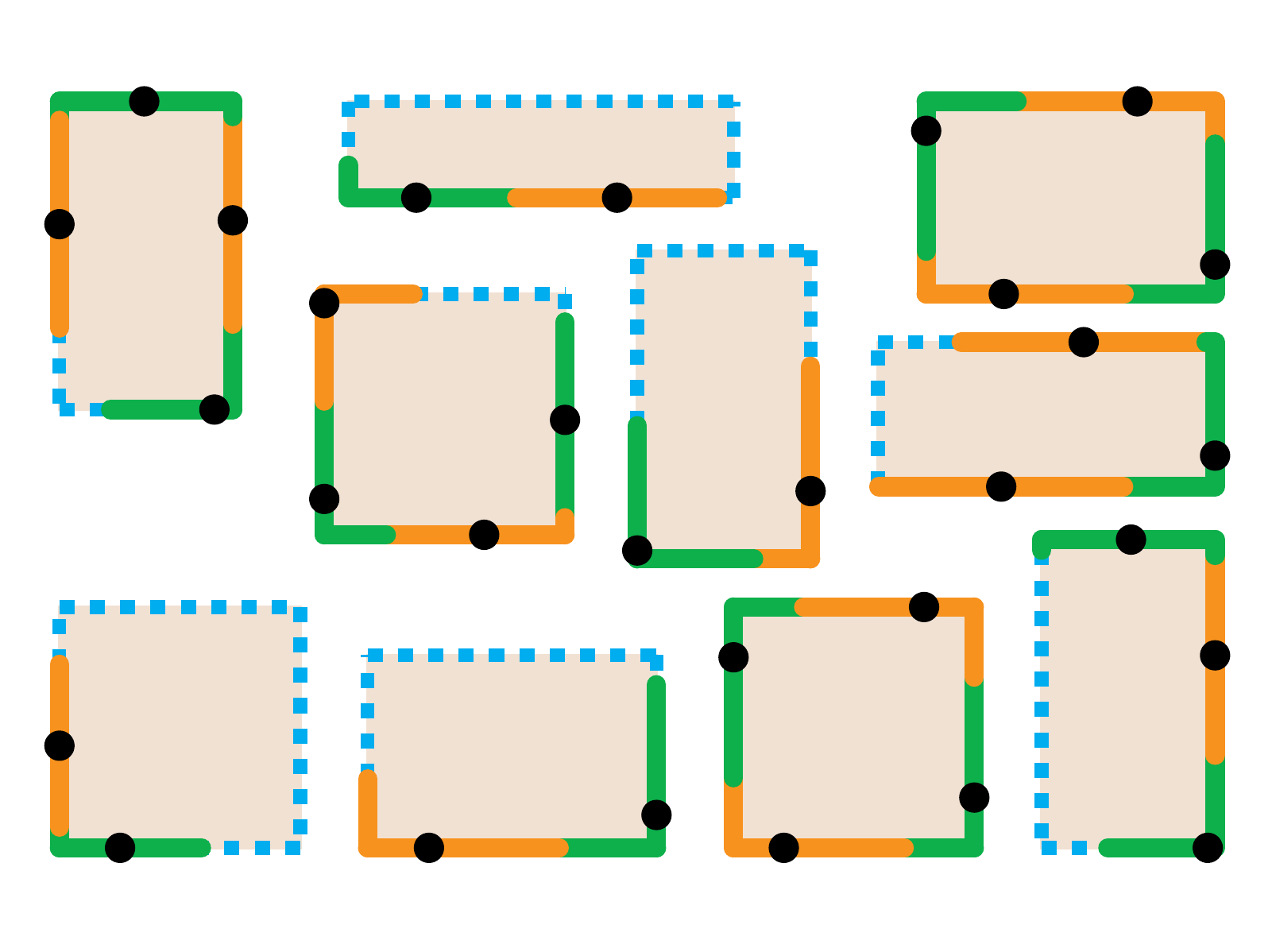}
		\fi
    \iffull {} \fi
    \caption{\label{fig:mpsc-example} 
    An example problem instance when $m = 10$ and $n = 30$. The black dots
		indicate deployed robot locations; the green and orange p-chains indicate
		the coverage.
		}
    \iffull {} \fi
\end{figure}

\iffull
\begin{table}[ht!]
    \centering
    \begin{footnotesize}
    \begin{tabular}{|c|c|c|c|c|c|c|} 
        \hline
        \diagbox{$m$}{$n$}       & $10^8  $ & $10^9   $ & $10^{10}$ & $10^{11}$ & $10^{12}  $ \\ \hline
        \rule{0pt}{2.5ex} $10^3$ & $0.001 $ & $0.001  $ & $0.001  $ & $0.001  $ & $0.001    $ \\ \hline
        \rule{0pt}{2.5ex} $10^4$ & $0.006 $ & $0.007  $ & $0.008  $ & $0.008  $ & $0.008    $ \\ \hline
        \rule{0pt}{2.5ex} $10^5$ & $0.075 $ & $0.088  $ & $0.102  $ & $0.107  $ & $0.106    $ \\ \hline
        \rule{0pt}{2.5ex} $10^6$ & $1.152 $ & $1.442  $ & $1.508  $ & $1.652  $ & $1.617    $ \\ \hline
        \rule{0pt}{2.5ex} $10^7$ & $13.963$ & $17.281 $ & $18.796 $ & $20.354 $ & $20.627   $ \\ \hline
        \rule{0pt}{2.5ex} $10^8$ & NA       & $176.115$ & $223.186$ & $227.250$ & $230.000  $ \\ \hline
    \end{tabular}
		\end{footnotesize}
    \caption{\label{eval:mpsc} \algoMRSimple~running time (seconds)}
\end{table}

To empirically verify the asymptotic running time upper bounds of 
\algoMRSimple, we plot the running time over $m$ for a fixed value of 
$n =10^{12}$. From the result (Fig.~\ref{fig:mpsc:mfixn}) it may be 
observed that the asymptotic running time appears to be tight. We point
out that the $O(\sum_{1\le i \le m} M_i + n)$ part of the overall 
running time $O(m(\log n + \log m) + \sum_{1\le i \le m} M_i + n)$ turns 
out to be rather insignificant (at least up to $m = 10^8$ and $n = 10^{12}$) 
and is subsequently ignored. The same applies to other algorithms as 
well. 
\begin{figure}[ht!]
    \vspace*{-2mm}
    \centering
    \includegraphics[keepaspectratio, scale=0.85]{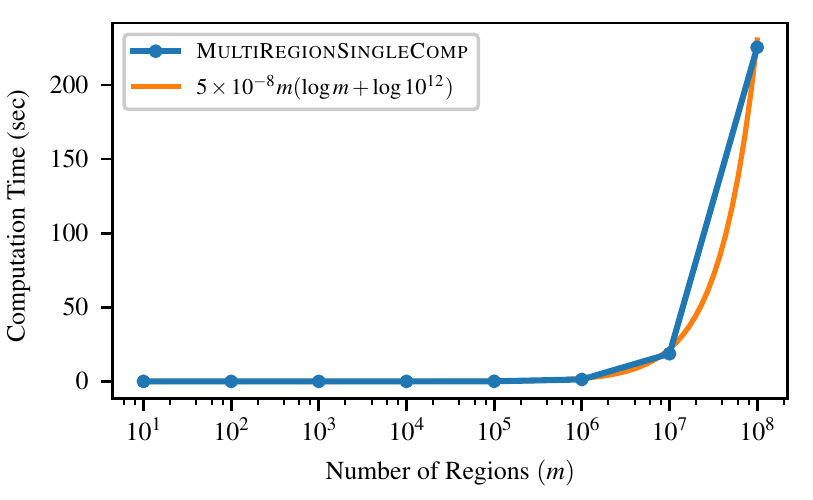}
    \vspace*{-4mm}
    \caption{\label{fig:mpsc:mfixn}Running time: \algoMRSimple 
		v.s. $O(m (\log m + \log 10^{12}))$.}
    \vspace*{-3mm}
\end{figure}

A similar study of checking the running time dependency over $n$ was 
carried out as well but did not show a tight dependency of the running 
time over $\log n$. This is because the $O(m\log n)$ part (from the 
$\mathbf{while}$ loop in \algoMRSimple) is dominated by the 
$O(m\log m)$ part (from the enhanced $\mathbf{while}$ loop with bisection). 
\fi


For the case of a single perimeter with multiple components, a random 
polygon is generated on which $2q$ points are randomly sampled that 
yield $q$ segments (that form the perimeter) and $q$ gaps. 
\iffull
Some example instances and the optimal solutions are illustrated in 
Fig.~\ref{fig:more-spmc-ex}.
\fi
The computation time for various $q$ and 
$n$ combinations is given in Table~\ref{eval:spmc}.
\iffull
\begin{figure}[ht!]
    \centering
		\includegraphics[keepaspectratio, scale=0.4]{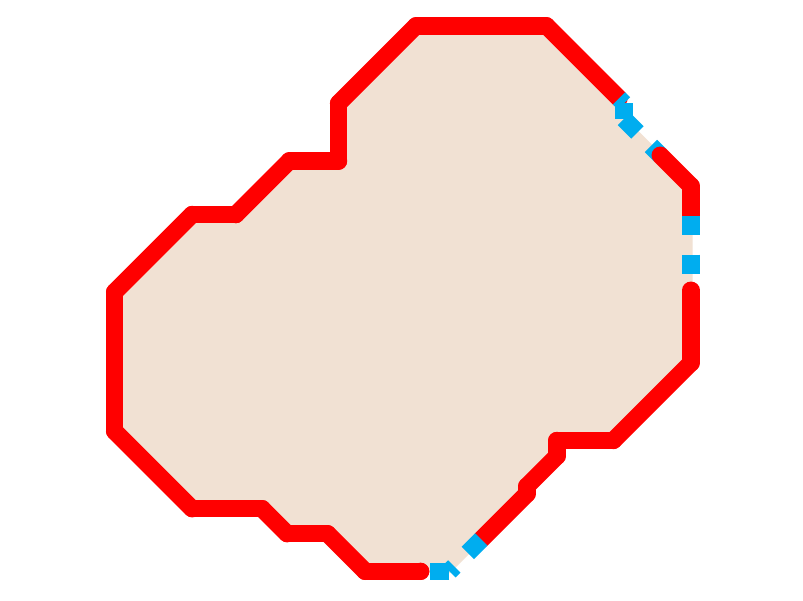}
    \includegraphics[keepaspectratio, scale=0.4]{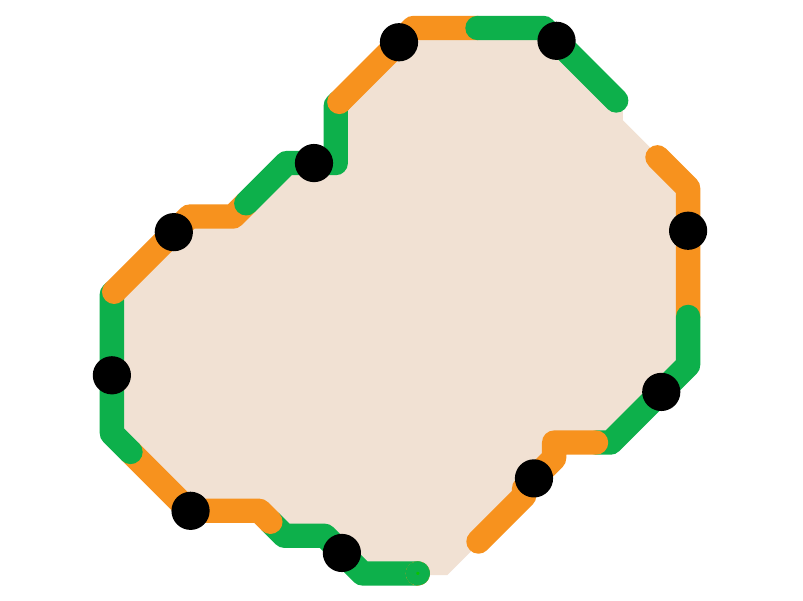} 
    \vspace*{2mm} \\
    \includegraphics[keepaspectratio, scale=0.4]{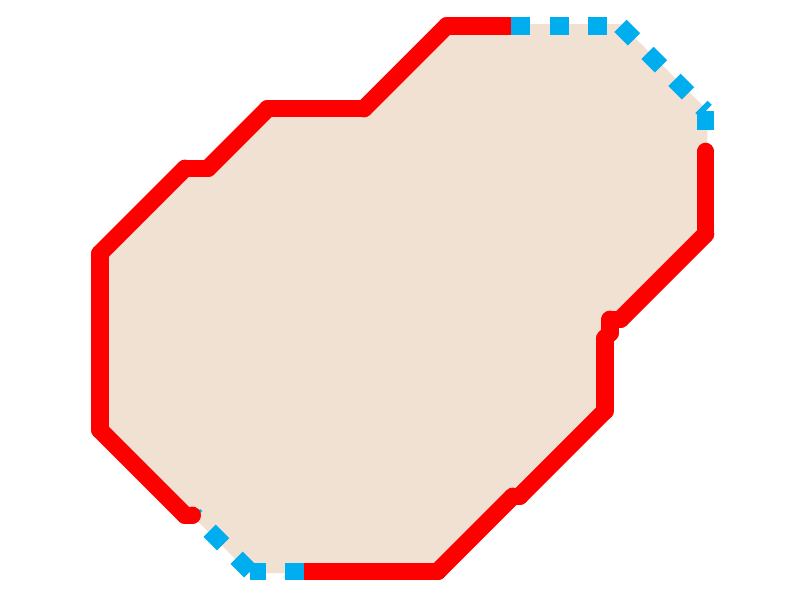}
    \includegraphics[keepaspectratio, scale=0.4]{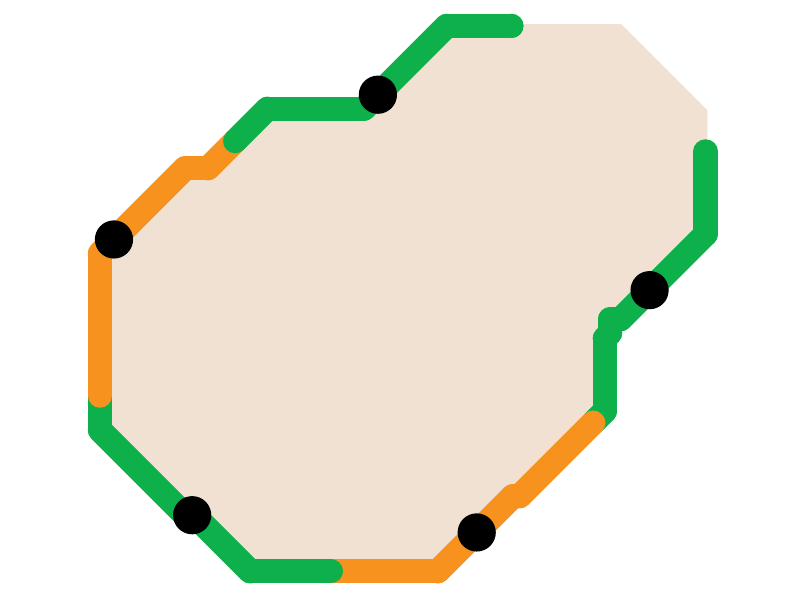}
    \vspace*{2mm} \\
    \includegraphics[keepaspectratio, scale=0.4]{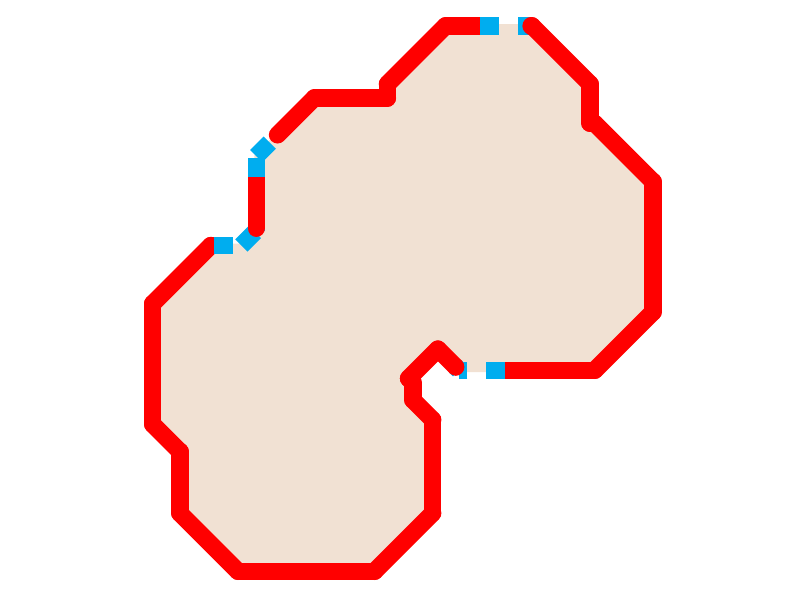}
    \includegraphics[keepaspectratio, scale=0.4]{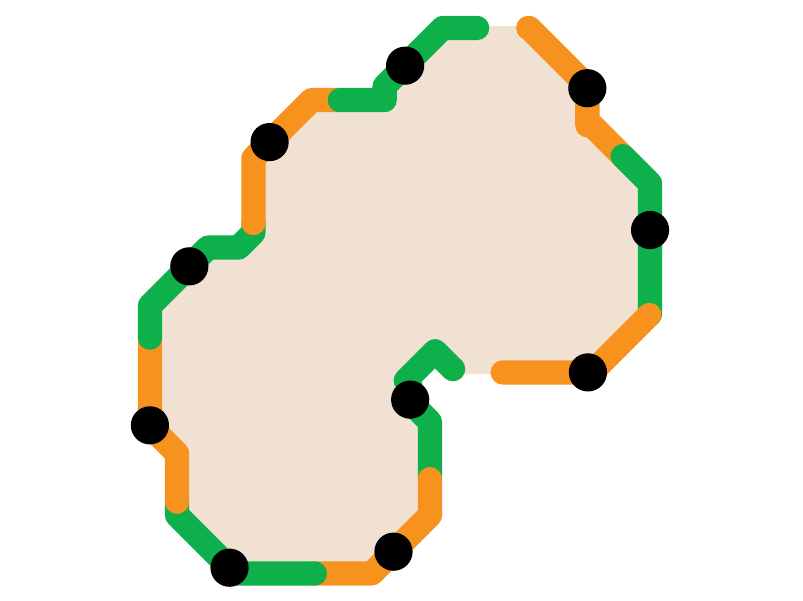} \\
    \vspace*{2mm}
    \includegraphics[keepaspectratio, scale=0.4]{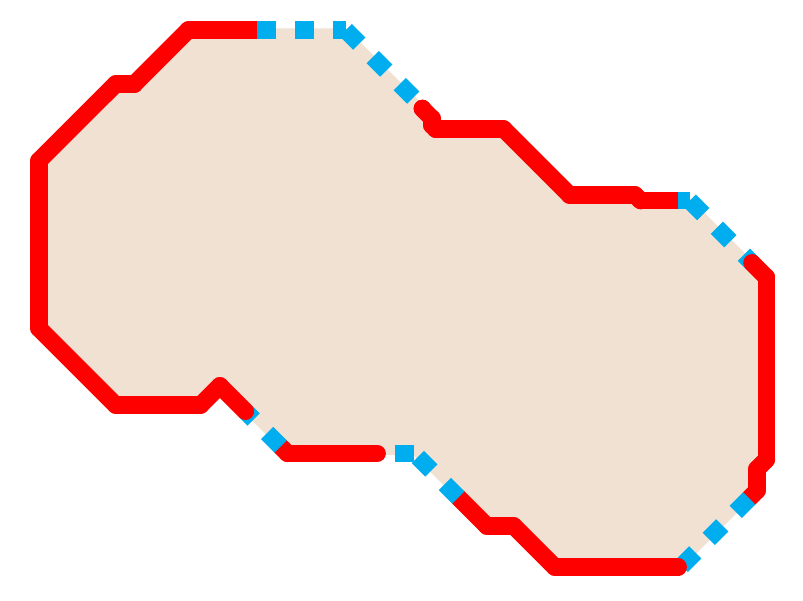}
    \hspace{5mm}
    \includegraphics[keepaspectratio, scale=0.4]{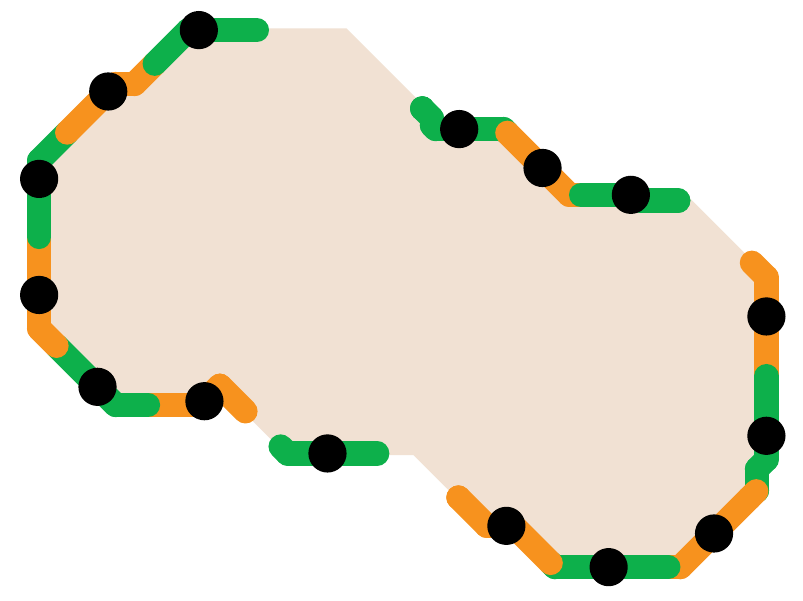}
    \caption{\label{fig:more-spmc-ex} Four examples for the case of single 
		region with multiple components. Problem parameters 
		are as follows: first row, $q = 3$, $n = 10$; second row, $q = 2$, 
		$n = 5$; third row, $q = 4$, $n = 10$; fourth row, $q = 5$, $n = 15$.} 
\end{figure}
\fi

\begin{table}[ht!]
    \iffull {} \fi
    \footnotesize
    \centering
    \begin{tabular}{|c|c|c|c|c|c|c|} 
        \hline
        \diagbox{$q$}{$n$}       & $10^1   $ & $10^2   $ & $10^3  $  & $10^4   $ & $10^5$   \\ \hline       
        \rule{0pt}{2.5ex} $10^2$ & $0.013  $ & $0.015  $ & $0.016 $  & $0.016  $ & $0.017$  \\ \hline   
        \rule{0pt}{2.5ex} $10^3$ & $1.363  $ & $1.595  $ & $1.622 $  & $1.634  $ & $1.641$  \\ \hline   
        \rule{0pt}{2.5ex} $10^4$ & $159.404$ & $188.497$ & $210.492$ & $212.473$ & $212.780$\\ \hline   
    \end{tabular}
    \iffull {} \fi
    \caption{\label{eval:spmc} \algoSRG~computation time (seconds)}
\end{table}

\iffull
For \algoSRG, the dependency of the running time over $q^2\log q$ 
appears to be tight (see Fig.~\ref{fig:spmc:qfixn}).
\begin{figure}[ht!]
    \centering
    \includegraphics[keepaspectratio, scale=0.85]{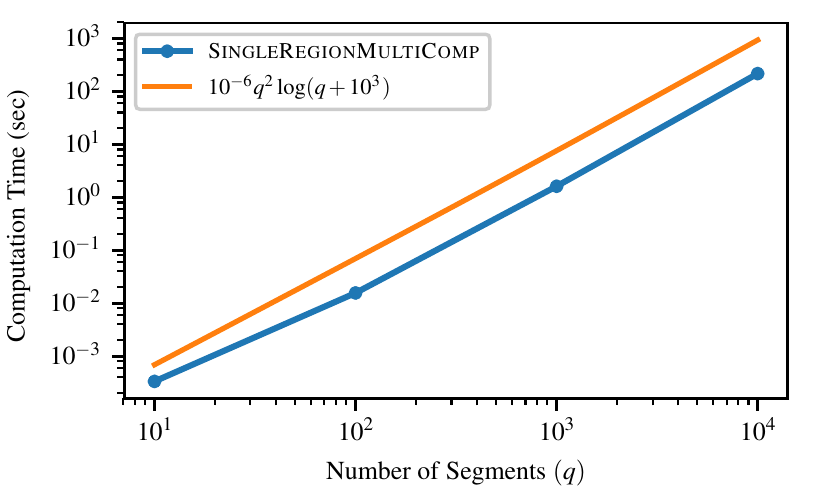}
    \vspace*{-4mm}
    \caption{\label{fig:spmc:qfixn}Running time of \algoSRG 
		v.s. $O(q^2\log(q + 10^3))$.}
    \vspace*{-2mm}
\end{figure}
\else
\fi


For multiple perimeters containing multiple components, $m$ polygons 
are created with $len(\partial R_i)$ randomly distributed in $[1, 10]$. 
For setting $q_i$, we fix a $q$ and let $q_i = q(0.5 + random(0, 1))$. 
Representative computation results of \algoMRG are listed in 
Table~\ref{eval:mpmc}.
\iffull
\begin{table}[ht!]
    \footnotesize
    \centering
    \begin{tabular}{|c|c|c|c|c|c|c|} 
        \hline
        \multirow{2}{*}{$q$} & \multirow{2}{*}{$n$} & \multicolumn{5}{|c|}{$m$} \\ \cline{3-7}
        \rule{0pt}{2.5ex} & & $10$ & $20$ & $30$ & $40$ & $50$ \\ \hline
        \rule{0pt}{2.5ex} $10^1$ & $10^2$ & $ 0.015$ & $ 0.027$ & $ 0.039$ & $ 0.045$ & $ 0.054$ \\ \hline
        \rule{0pt}{2.5ex} $10^1$ & $10^3$ & $ 0.047$ & $ 0.063$ & $ 0.076$ & $ 0.091$ & $ 0.108$ \\ \hline
        \rule{0pt}{2.5ex} $10^2$ & $10^2$ & $ 1.492$ & $ 2.784$ & $ 4.168$ & $ 5.404$ & $ 6.444$ \\ \hline
        \rule{0pt}{2.5ex} $10^2$ & $10^3$ & $ 2.191$ & $ 3.771$ & $ 5.523$ & $ 7.707$ & $ 9.369$ \\ \hline
        \rule{0pt}{2.5ex} $10^2$ & $10^4$ & $ 7.105$ & $ 9.619$ & $11.369$ & $12.760$ & $15.107$ \\ \hline
    \end{tabular}
    \caption{\label{eval:mpmc} \algoMRG~computation time (seconds)}
\end{table}
\fi

\iffull
As for running time, Fig.~\ref{fig:mpmc:m} shows 
the dependency on the number of regions $m$ appears to be linear with 
$q$ fixed (recall we set $q_i = q(0.5 + random(0,1))$). This is tight
in viewing the main running time of \algoMRG which is 
$O((\sum_{1\le i \le m} q_i^2) \log(n + \sum_{1\le i \le m} q_i))$; if
$q_i$ is fixed, then the time is linear with respect to $m$. An example 
computation result for $m = 3$ is illustrated in Fig.~\ref{fig:mpmc-ex}. 

\begin{figure}[ht!]
    \centering
    \includegraphics[keepaspectratio, scale=0.85]{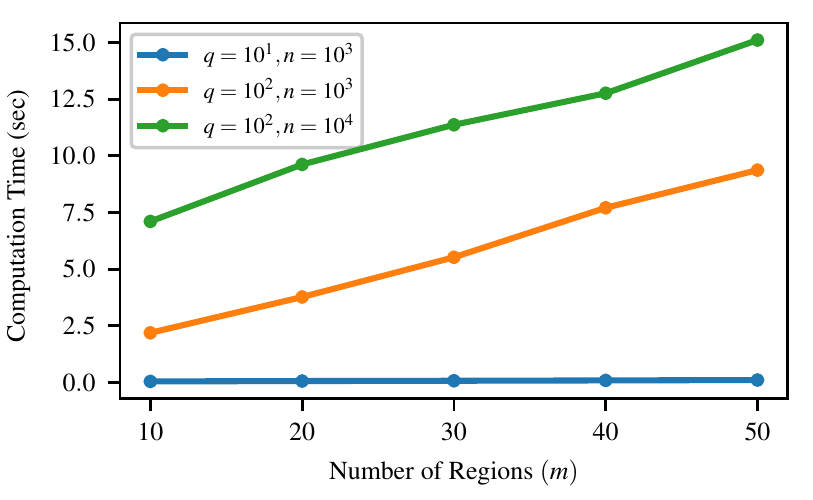}
    \caption{\label{fig:mpmc:m}Running time of \algoMRG 
		v.s. the number of regions.}
\end{figure}

\begin{figure}[ht!]
    \centering
    \includegraphics[keepaspectratio, scale=0.25]{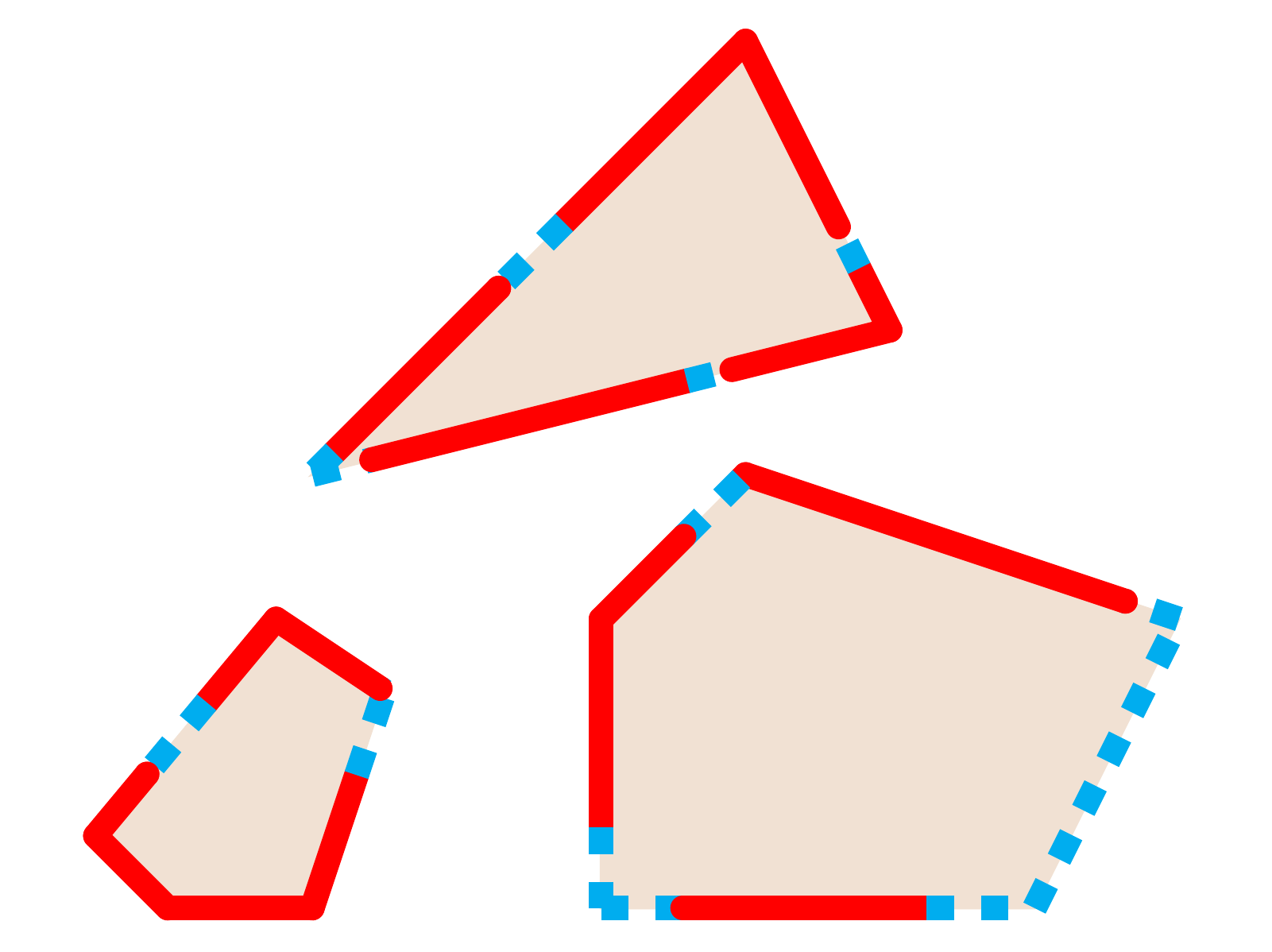}
    \includegraphics[keepaspectratio, scale=0.25]{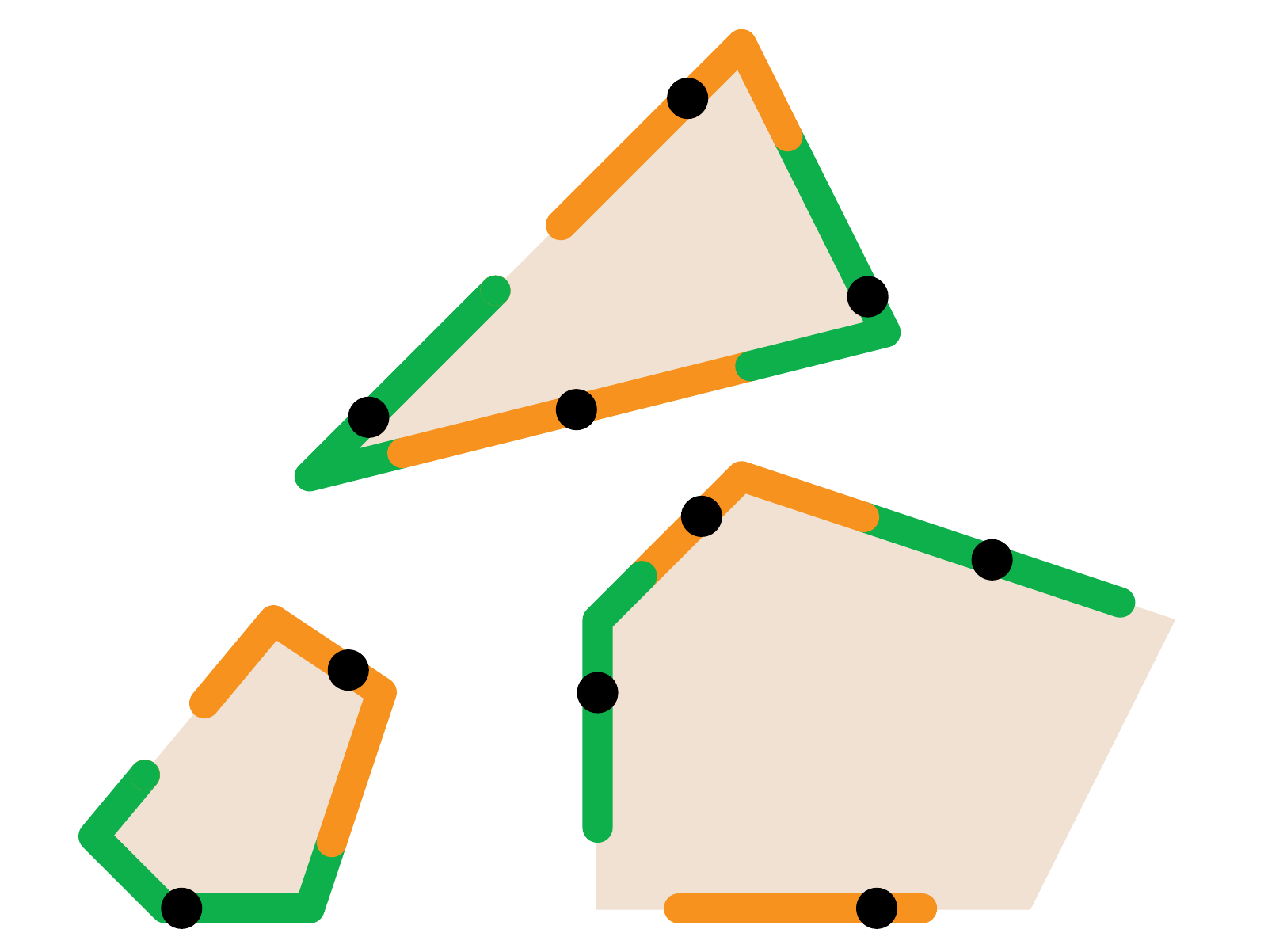}
    \caption{\label{fig:mpmc-ex} An example instance when $m = 3$, $n = 10$.} 
\end{figure}
\fi

\subsection{Two Applications Scenarios}
\noindent\textbf{Securing a perimeter}. As a first application, consider
a situation where a crime has just been committed at the Edinburgh 
Castle (see Fig.~\ref{fig:edinburgh}). The culprit remains in the confines 
of the castle but is mixed within many guests at the scene. As the 
situation is being investigated and suppose that the brick colored 
buildings are secured, guards (either personnel or a number of drones) may 
be deployed to ensure the culprit does not escape by climbing down the 
castle walls. Using \algoSRG, a deployment plan can be quickly computed 
given the amount of resources at hand so that each guard only needs to 
secure a minimum length along the castle walls. Fig.~\ref{fig:edinburgh} 
shows the optimal deployment plan for $15$ guards. 

\begin{figure}[ht]
	\iffull {} \fi
	\begin{center}
		\begin{overpic}[width={\iffull 2.5in \fi},tics=5]{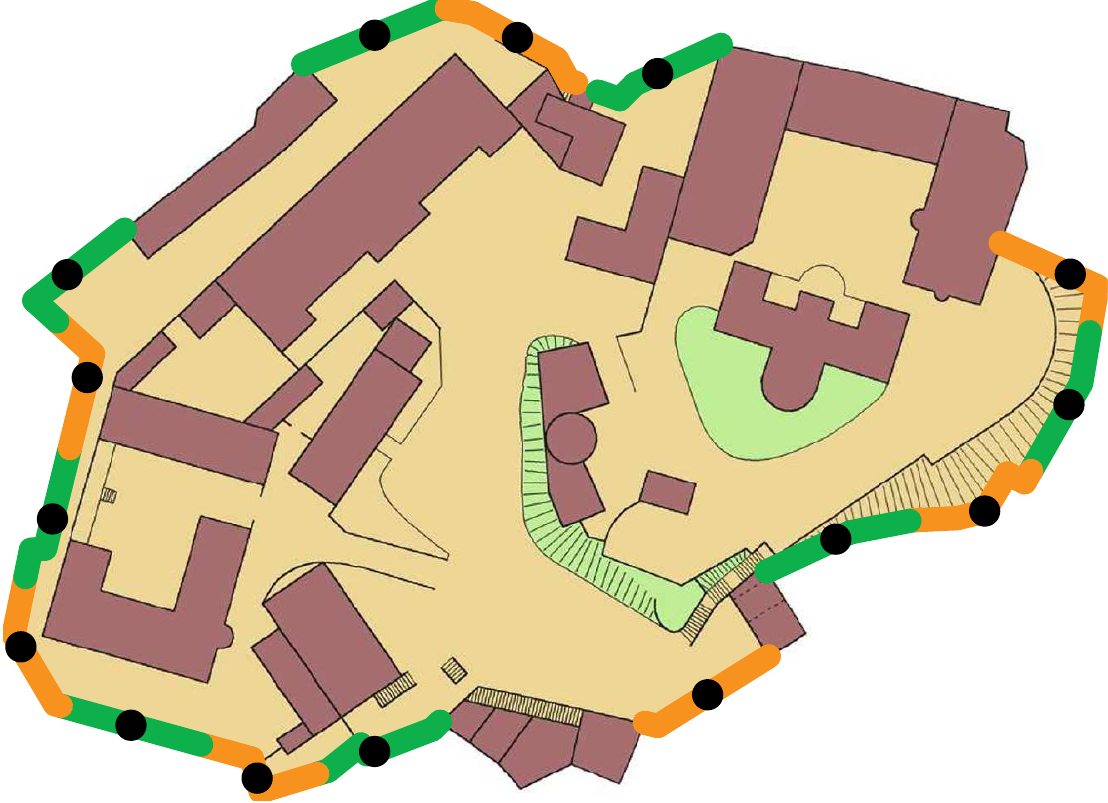}
		\end{overpic}
	\end{center}
	\iffull {} \fi
	\caption{\label{fig:edinburgh} Optimal deployment of $15$ guards around 
	walls of the Edinburgh Castle. The brick colored structures are buildings 
	that create gaps along the boundary.}
	\iffull {} \fi
\end{figure}

\iffull
Then, Fig.~\ref{fig:more-edinburgh} shows the deployment plan for $n = 5, 10, 
20, 30$ guards. As the number of guards changes from $5$ to $10$, the gap on the
lower left side is no longer covered due to the availability of more guards.
Similarly, as the number of guards changes from $20$ to $30$, the very small 
gap on the top no longer needs to be covered.  

\begin{figure}[ht!]
	\begin{center}
		\begin{overpic}[width={\iffull 2.5in \else 2.05in \fi},tics=5]{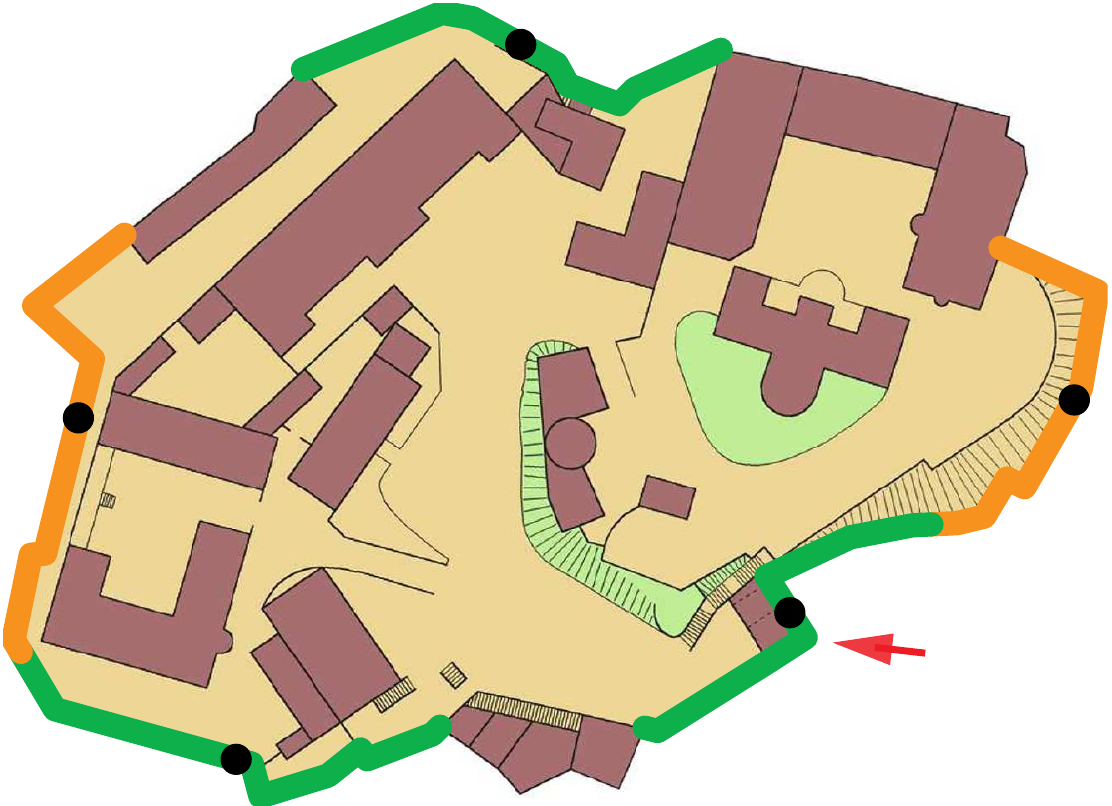}
		\end{overpic}
    \end{center}
		\vspace*{2mm}
	\begin{center}
		\begin{overpic}[width={\iffull 2.5in \else 2.05in \fi},tics=5]{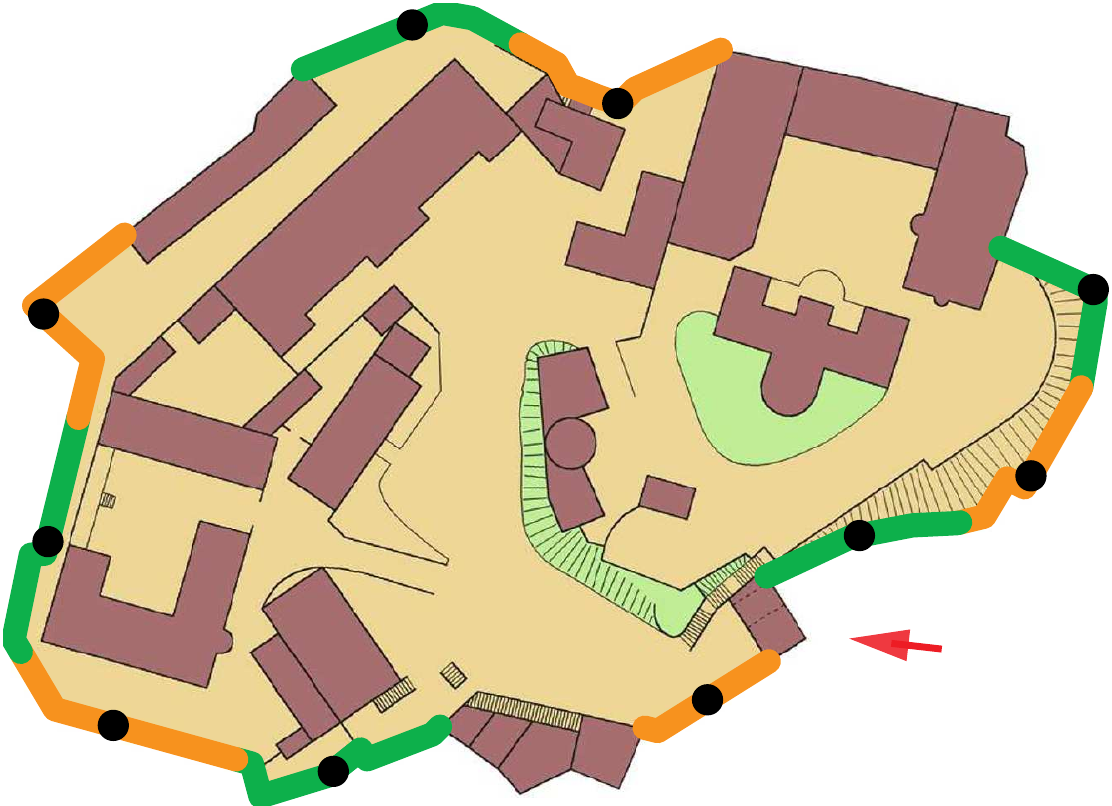}
		\end{overpic}
    \end{center}
		\vspace*{2mm}
	\begin{center}
		\begin{overpic}[width={\iffull 2.5in \else 2.05in \fi},tics=5]{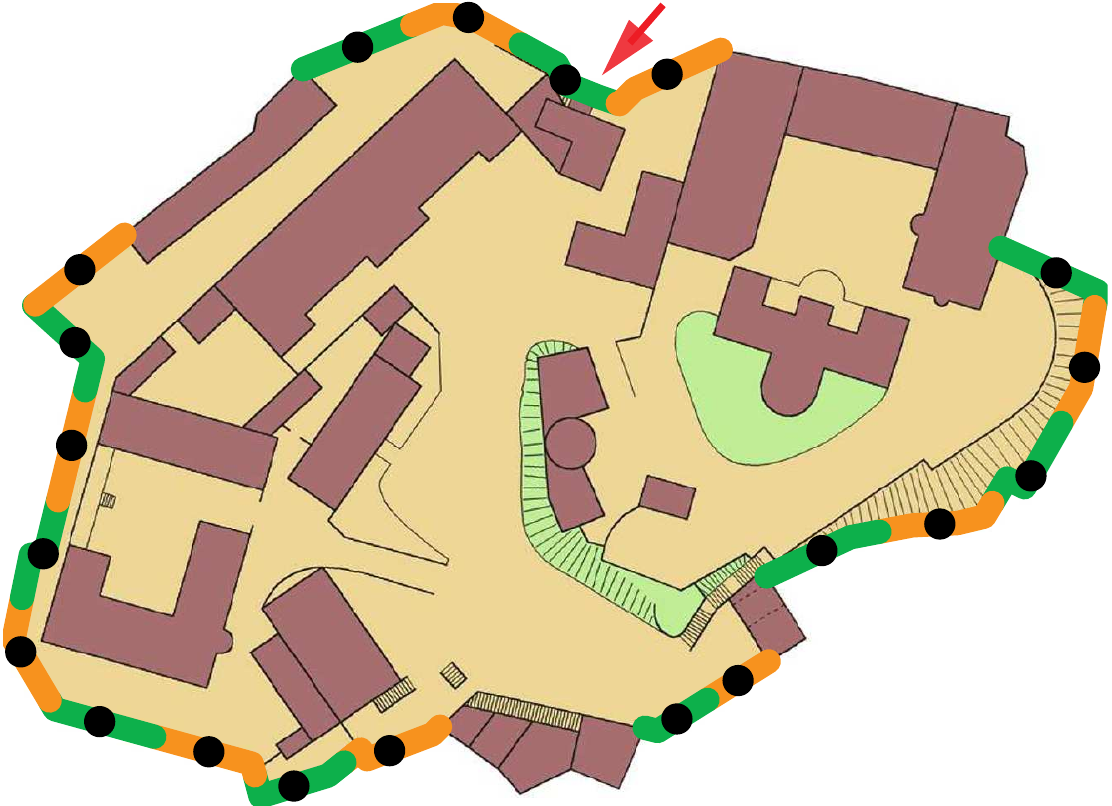}
		\end{overpic}
    \end{center}
		\vspace*{2mm}
	\begin{center}
		\begin{overpic}[width={\iffull 2.5in \else 2.05in \fi},tics=5]{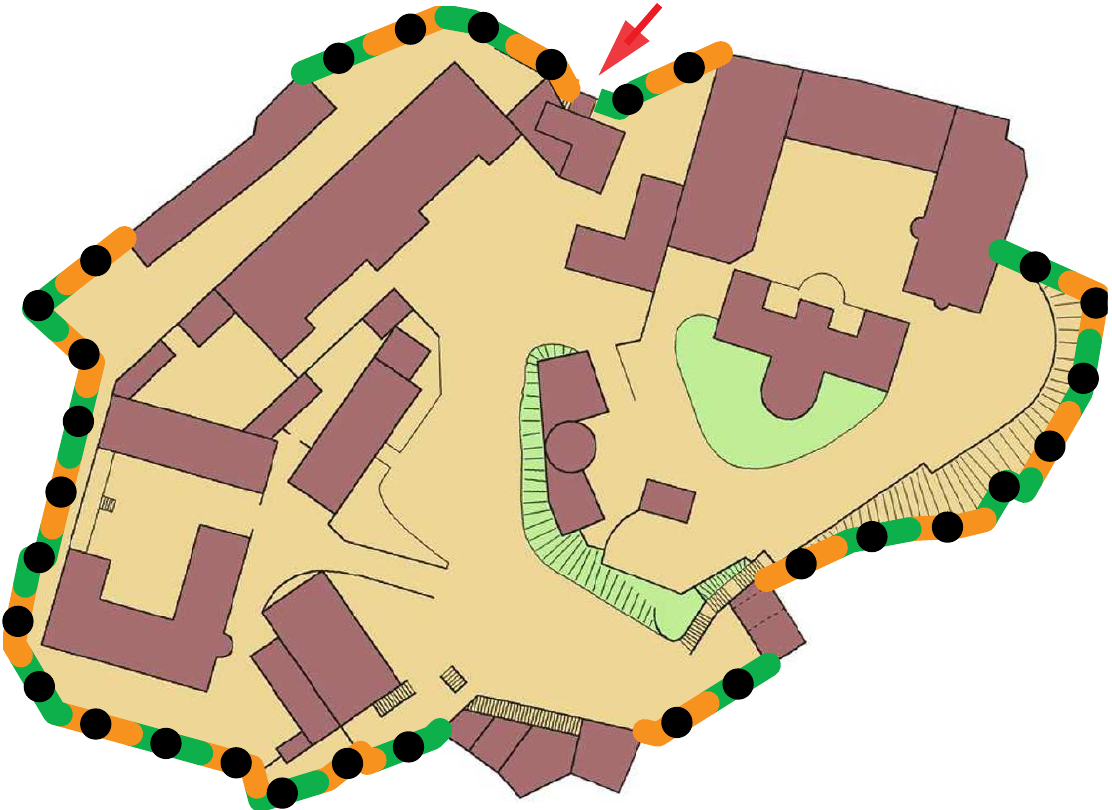}
		\end{overpic}
	\end{center}
	\vspace*{-2mm}
	\caption{\label{fig:more-edinburgh} Optimal deployment of $5, 10, 20, 30$ 
	guards around the Edinburgh Castle.}
\end{figure}
\else
\fi

\noindent\textbf{Fire monitoring}. In a second application, consider 
Fig.~\ref{fig:forest} where a forest fire has just been put out in 
multiple regions. As there is still some chance that the fire may 
rekindle and spread, for prevention, a team of firefighters is to be 
deployed to watch for the possible spreading of the fire. Here, in 
addition to using \algoMRG to compute optimal locations for deploying 
the firefighters, we also generate minimum time trajectories for the 
firefighters to reach their target locations while avoiding going 
through the dangerous forests. This is done via solving a bottleneck 
assignment problem \cite{burkard1999linear}.
Note that the lake region creates gaps that cannot be traveled by the 
firefighters; this can be handled by making these gaps infinitely large. 
Fig.~\ref{fig:forest} shows the optimal locations for $34$ firefighters. 
Animations of the deployment process and other test cases can be found 
in the accompanying video. 

\begin{figure}[ht]
	\iffull {} \fi
	\begin{center}
		\begin{overpic}[width={\iffull 2.5in \fi},tics=5]{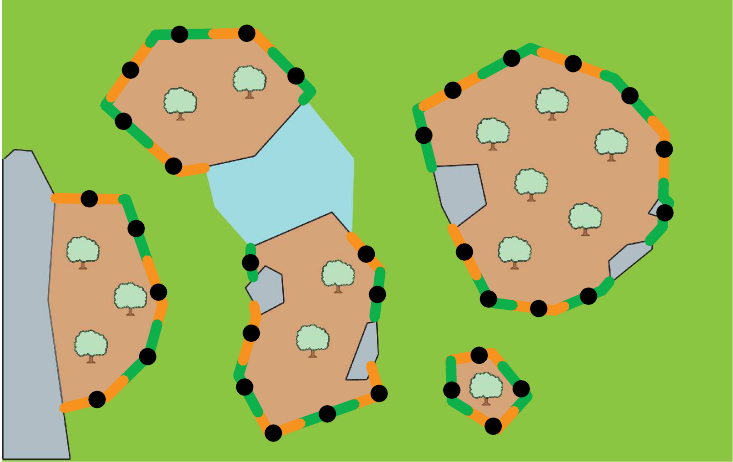}
		\end{overpic}
	\end{center}
	\iffull {} \fi
	\caption{\label{fig:forest}  Optimal deployment of $34$ firefighters for 
	forest fire rekindling prevention.}
	\iffull {} \fi
\end{figure}

\iffull
Additional computational results for the forest fire monitoring
case is illustrated in Fig.~\ref{fig:more-forest}. Behavior similar to that 
from the castle case can be observed here, e.g., from $30$ to $40$ 
guards, the small gap on the right is no longer covered. 

\begin{figure}[ht!]
	\begin{center}
		\begin{overpic}[width={\iffull 2.5in \else 2.4in \fi},tics=5]{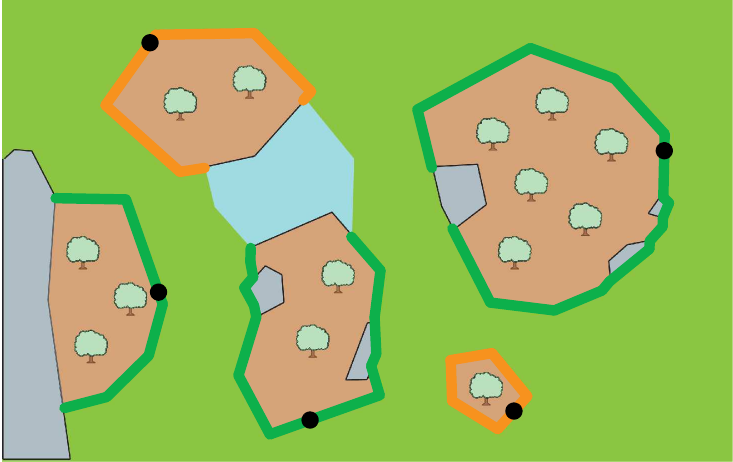}
		\end{overpic}
    \end{center}
	\begin{center}
		\begin{overpic}[width={\iffull 2.5in \else 2.4in \fi},tics=5]{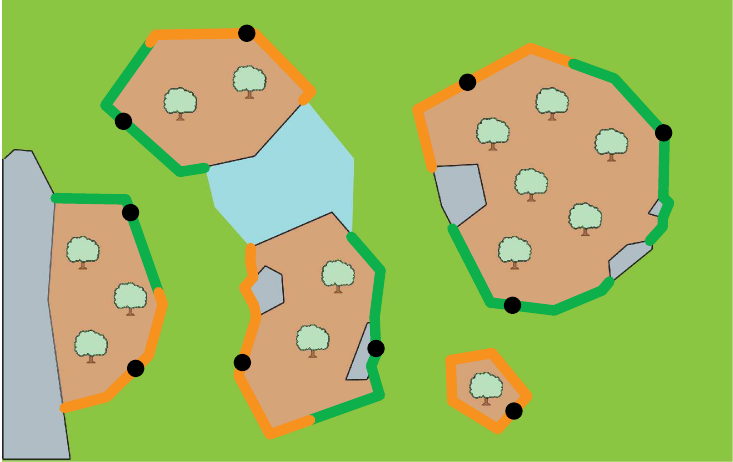}
		\end{overpic}
    \end{center}
	\begin{center}
		\begin{overpic}[width={\iffull 2.5in \else 2.4in \fi},tics=5]{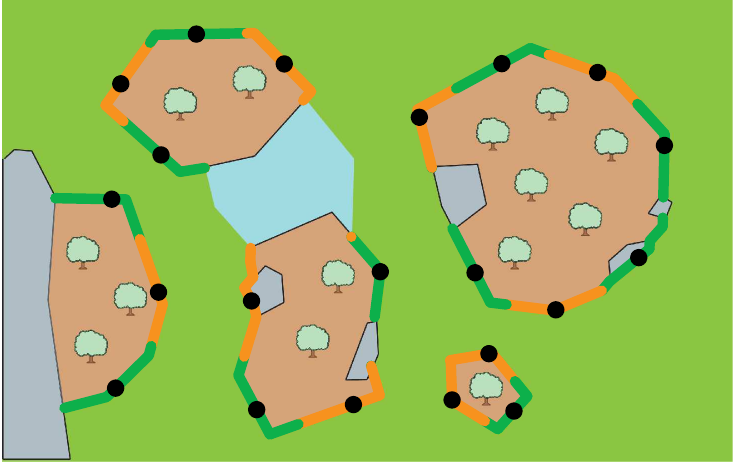}
		\end{overpic}
    \end{center}
	\begin{center}
		\begin{overpic}[width={\iffull 2.5in \else 2.4in \fi},tics=5]{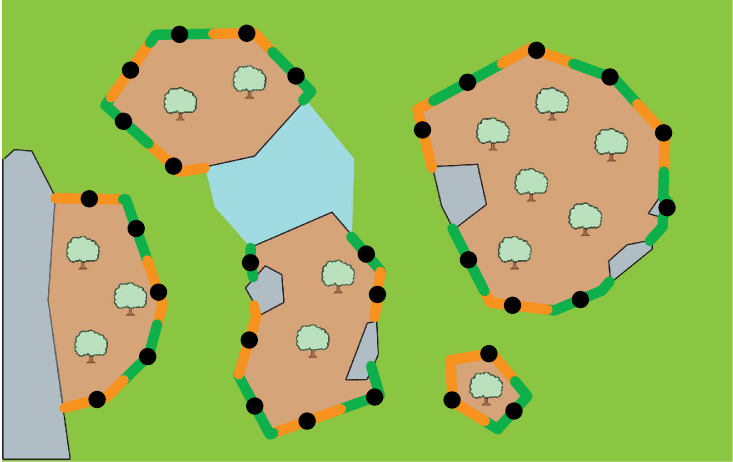}
		\end{overpic}
    \end{center}
	\begin{center}
		\begin{overpic}[width={\iffull 2.5in \else 2.4in \fi},tics=5]{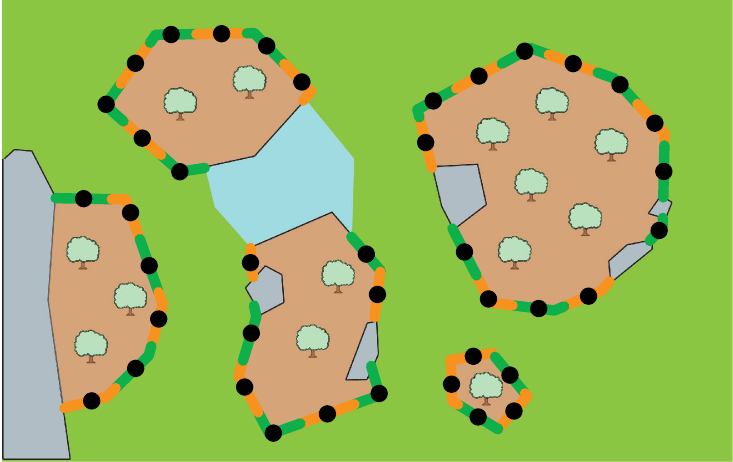}
		\end{overpic}
	\end{center}
	\vspace*{-2mm}
	\caption{\label{fig:more-forest}  Optimal deployment of $5, 10, 20, 30, 40$ firefighters for 
	forest fire monitoring.}
	\vspace*{-3mm}
\end{figure}
\else
\fi

\section{Conclusion and Discussion}\label{section:conclusion}
In this paper, we propose the \opg problem to model the allocation of 
a large number of robots to cover complex 1D topological domains with 
optimality guarantees. For all variants under the \opg formulation 
umbrella, we have developed highly efficient algorithms for solving 
\opg exactly. In addition to rigorous proofs backed by formal analysis, 
extensive computational experiments further confirm the effectiveness of 
these algorithms. Moreover, practical relevance of \opg is demonstrated 
through the integration of \opg into realistic task (assignment) and motion 
planning scenarios. 

The study raises many additional interesting open questions; we mention 
a few here. 
First, the approach taken in this work is a {\em centralized} one where 
decision is made at the global level. It would be highly interesting to 
explore whether the same can be achieved with {\em decentralized} methods,
which have many advantages. For example, it may be the case that the 
gaps along the boundaries are not known {\em a priori} and must be measured
by the robots. In such cases, a centralized plan can be hard to come by. 
Second, as mentioned in Section~\ref{section:problem}, the 
current \opg formulation assumes that the robots are confined to the 
boundaries $\partial \R$, which is one of many possible choices 
in terms of the robots' sensing and/or motion capabilities. In future study,
we plan to examine additional practical robot sensing and motion models. 
Third, as exact optimal algorithms are emphasized here, issues including 
uncertainty and robustness have not been touched in the current treatment, 
which are important elements when it comes to the deployment of a robotic 
swarm to tackle real-world challenges. 

\textbf{Acknowledgments}.
This work is supported by NSF grants IIS-1617744, IIS-1734419, and IIS-1845888. 
Opinions or findings expressed in this paper do not necessarily reflect 
the views of the sponsors. 


\bibliographystyle{IEEEtran}
\bibliography{bib}

\end{document}